\documentclass{article}

\usepackage{fullpage}
\usepackage{amsmath,amsfonts,amsthm,amssymb}
\usepackage{url}            % simple URL typesetting
\usepackage{nicefrac}       % compact symbols for 1/2, etc.
\usepackage{microtype}
\usepackage{graphicx}
\usepackage{subfig}
\usepackage{booktabs} % for professional tables
\usepackage{caption}
\usepackage{algorithm}
\usepackage{algorithmic}
\newtheorem{theorem}{Theorem}[section]

\newtheorem{claim}[theorem]{Claim}

\newtheorem{proposition}[theorem]{Proposition}
\newtheorem{lemma}[theorem]{Lemma}
\newtheorem{corollary}[theorem]{Corollary}

\newtheorem{observation}[theorem]{Observation}
\newtheorem{fact}[theorem]{Fact}

\newtheorem{definition}[theorem]{Definition}

\newtheorem{assumption}[theorem]{Assumption}

\newcommand{\newreptheorem}[2]{%
\newenvironment{rep#1}[1]{%
 \def\rep@title{#2 \ref{##1}}%
 \begin{rep@theorem}}%
 {\end{rep@theorem}}}

\newreptheorem{theorem}{Theorem}
\newreptheorem{lemma}{Lemma}
\newreptheorem{proposition}{Proposition}
\newreptheorem{claim}{Claim}
\newreptheorem{corollary}{Corollary}
\newreptheorem{mainlemma}{Main Lemma}

\newtheorem{remark}[theorem]{Remark}

\usepackage[capitalise]{cleveref}
\usepackage{dsfont}

\numberwithin{equation}{section}

\def\S{{\mathbb S}}

\newcommand{\A}{\mathcal{A}}

\newcommand{\K}{\ensuremath{\mathcal K}}

\def\regret{\mbox{{Regret\ }}}

\newcommand{\ignore}[1]{}

\makeatletter
\newcommand{\neutralize}[1]{\expandafter\let\csname c@#1\endcsname\count@}
\makeatother

\newtheorem*{theorem*}{Theorem}
\newtheorem*{lemma*}{Lemma}
\newtheorem*{corollary*}{Corollary}
\newtheorem*{proposition*}{Proposition}
\newtheorem*{claim*}{Claim}
\newtheorem*{fact*}{Fact}
\newtheorem*{observation*}{Observation}

\theoremstyle{definition}

\newtheorem*{definition*}{Definition}
\newtheorem*{remark*}{Remark}
\newtheorem*{example*}{Example}

 \theoremstyle{plain}
\DeclareMathAlphabet{\mathbfsf}{\encodingdefault}{\sfdefault}{bx}{n}

\newcommand{\norm}[1]{\left\lVert#1\right\rVert}

\renewcommand{\O}{O}

\newcommand{\E}{\mathop{\mathbb{E}}}

\newcommand{\poly}{\mathrm{poly}}

\newcommand{\reals}{\mathbb{R}}
\newcommand{\eps}{\varepsilon}

\renewcommand{\leq}{~\le~}
\renewcommand{\geq}{~\ge~}

\let\oldtfrac\tfrac
\renewcommand{\tfrac}[2]{\smash{\oldtfrac{#1}{#2}}}

\let\nablaold\nabla
\renewcommand{\nabla}{\nablaold\mkern-2.5mu}

\title{Non-Stochastic Control with Bandit Feedback}

\author{
  Paula Gradu \\
  Department of Mathematics\\
  Princeton University\\
  \texttt{pgradu@princeton.edu} \\
  \AND
  John Hallman \\
  Department of Mathematics\\
  Princeton University\\
  \texttt{hallman@princeton.edu} \\
  \And
  Elad Hazan \\
  Department of Computer Science\\
  Princeton University \\
  \texttt{ehazan@cs.princeton.edu}
}

\title{Non-Stochastic Control with Bandit Feedback}
\date{\today}
\author{  Paula Gradu$^{1, 3}$ \qquad  John Hallman$^{1, 3}$ \qquad  Elad Hazan$^{2, 3}$ \\
  $^1$ Department of Mathematics, Princeton University \\
  $^2$ Department of Computer Science, Princeton University \\
  $^3$ Google AI Princeton \\
  \texttt{\{pgradu,hallman,ehazan\}@princeton.edu }  \\
 }

\begin{document}

\maketitle

\begin{abstract}
We study the problem of controlling a linear dynamical system with adversarial perturbations where the only feedback available to the controller is the scalar loss, and the loss function itself is unknown. For this problem, with either a known or unknown system, we give an efficient sublinear regret algorithm.  
The main algorithmic difficulty is the dependence of the loss on past controls. To overcome this issue,  we  propose an efficient algorithm for the general setting of bandit convex optimization for loss functions with memory, which may be of independent interest. 
\end{abstract}

\section{Introduction}

The fields of Reinforcement Learning (RL), as well as its differentiable counterpart of Control, formally model the setting of learning through interaction in a reactive environment. The crucial component in RL/control that allows learning is the feedback, or reward/penalty, which the agent iteratively observes and reacts to. 

While some signal is necessary for learning, different applications have different feedback to the learning agent. In many reinforcement learning and control problems it is unrealistic to assume that the learner has feedback for actions other than their own. One example is in game-playing, such as the game of Chess, where a player can observe the adversary's move for their own choice of play, but it is unrealistic to expect knowledge of the adversary's play for any possible move. This type of feedback is commonly known in the learning literature as ``bandit feedback". 

Learning in Markov Decision Processes (MDP) is a general and difficult problem for which there are no known algorithms that have sublinear dependence on the number of states. For this reason we look at structured MDPs, and in particular the model of control in Linear Dynamical Systems (LDS), a highly structured special case that is known to admit more efficient methods as compared to general RL. 

In this paper we study learning in linear dynamical systems with bandit feedback. This generalizes the well-known Linear Quadratic Regulator to systems with only bandit feedback over any convex loss function. Further, our results apply to the non-stochastic control problem which allows for adversarial perturbations and adversarially chosen loss functions, even when the underlying linear system is unknown.  
\subsection{Our Results}

We give the first sublinear regret algorithm for controlling a linear dynamical system with bandit feedback in the non-stochastic control model. Specifically, we consider the case in which the underlying system is linear, but has potentially adversarial perturbations (that can model deviations from linearity), i.e.
\begin{equation} \label{eqn:shalom}
x_{t+1} = A x_t + B u_t + w_t ,
\end{equation}
where $x_t$ is the (observed) dynamical state, $u_t$ is a learner-chosen control and $w_t$ is an adversarial perturbation. The goal of the controller is to minimize a sum of sequentially revealed adversarial cost functions $c_t(x_t,u_t)$ over the state-control pairs that it visits. More precisely, the goal of the learner in this adversarial setting is to minimize regret compared to a class of policies $\Pi$: 

$$ \regret =  \sum_{t=1}^T c_t (x_t ,u_t) - \min_{\pi \in \Pi} \sum_{t=1}^T c_t(x^\pi_t , u^\pi_t), $$

where the cost of the benchmark is measured on the counterfactual state-action sequence $(x^\pi_t, u^\pi_t)$ that the benchmark policy in consideration visits, as opposed to the state-sequence visited by the the learner. The target class of policies we compare against in this paper are disturbance action controllers (DAC), % strongly stable\footnote{strong stability is rigorously defined in the appendix and not required for the derivation of our main results.} 
whose control is a linear function of past disturbances plus a stabilizing linear operator over the current state $u_t = K x_t + \sum_{i=1}^H M_i w_{t-i}$, for some history-length parameter $H$. This comparator class is known to be more general than the state-of-the-art in linear control: linear dynamical controllers (LDC). This choice is a consequence of recent advances in convex relaxation for control  \cite{agarwal2019online,agarwal2019logarithmic,hazan2019nonstochastic,simchowitz2020improper}. 

For the setting we consider, the controller can only observe the scalar $c_t(x_t,u_t)$, and  {\bf does not have access to the gradients or any other information about the loss}.
Our main results are efficient algorithms for the non-stochastic control problem which attain the following guarantees:
\begin{theorem}[Informal Statement]
For a \textbf{known} linear dynamical system where the perturbations $w_t$ (and convex costs $c_t$) are bounded and chosen by an  adversary, there exists an efficient algorithm that with bandit feedback generates an adaptive sequence of controls $\{u_t\}$ for which 
$$ \regret  = \mathcal{\widetilde{O}}(\poly(\texttt{natural-parameters})T^{3/4}) .$$
\end{theorem}
This theorem can be further extended to unknown systems: 
\begin{theorem}[Informal Statement]
For an \textbf{unknown} linear dynamical system where the perturbations $w_t$ (and convex costs $c_t$) are bounded and chosen by an  adversary, there exists an efficient algorithm that with bandit feedback generates an adaptive sequence of controls $\{u_t\}$ for which 
$$ \regret  = \mathcal{\widetilde{O}}(\poly(\texttt{natural-parameters})T^{3/4}) .$$
\end{theorem}

\paragraph{Techniques.}
To derive these results, we combine the convex relaxation technique of \cite{agarwal2019online} with the non-stochastic system identification method for environments with adversarial perturbations from \cite{hazan2019nonstochastic,simchowitz2019learning}. However, the former result relies on gradient based optimization methods, and it is non-trivial to apply gradient estimation techniques in this black-box zero-order information setting. The main difficulty stems from the fact that the gradient-based methods from non-stochastic control apply to {\it functions with memory}, and depend on the system state going back many iterations. The natural way of creating unbiased gradient estimates, such as in \cite{flaxman2004online}, have no way of accounting for functions with memory.

To solve this difficulty, we introduce an efficient algorithm for the setting of {\it bandit convex optimization with memory}. This method combines the gradient-based methods of \cite{anava2013online} with the unbiased gradient estimation techniques of \cite{flaxman2004online}. The naive way of combining these techniques introduces time dependencies between the random gradient estimators, as a direct consequence of the memory in the loss functions. To resolve this issue, we introduce an artificial intentional delay to the gradient updates and show that this delay has only a limited effect on the overall regret. 

\paragraph{Paper outline.}
After describing related work, we cover preliminaries and define notation in section \ref{sec:preliminaries}. In section \ref{sec:alg} we describe the algorithm for BCO with memory and the main theorem regarding its performance. We then introduce the bandit control setting in section \ref{sec:control}, and provide algorithms for known and unknown systems in sections \ref{sec:known} and \ref{sec:unknown} respectively, together with relevant theoretical results. We then present experimental results in section \ref{sec:experiments}.

\subsection{Related Work\label{ssec:prior_work}}

\paragraph{Reinforcement learning with bandit feedback.}  Online learning techniques for reinforcement learning were studied in \cite{even2009online} and generalized in \cite{yu2009markov}. Online learning for RL with bandit feedback was studied in \cite{neu2010online}. For general RL it is impossible to obtain regret bounds that are sublinear in the number of states, even with full feedback. This is the reason we focus on much more structured problem of control, where our regret bounds depend on the dimension despite an infinite number of states, even in the bandit setting.

\paragraph{Robust Control:} The classical control literature deals with  adversarial perturbations in the dynamics in a framework known as  $H_\infty$ control, see e.g. \cite{z2,z1}. In this setting, the controller solves for the best linear controller assuming worst case noise to come. This is different from the setting we study which minimizes regret on a per-instance basis. 

\paragraph{Learning to control stochastic LDS:} 
There has been a resurgence of literature on control of linear dynamical systems in the recent machine learning venues. The case of known systems was extensively studied in the control literature, see the survey \cite{z2}. Sample complexity and regret bounds for control (under Gaussian noise) were obtained in \cite{abbasi2011regret,dean2018regret,a2,mania2019certainty,cohen2019learning,anima1,anima2,anima3}. 
The works of \cite{abbasi2014tracking}, \cite{cohen2018online} and \cite{agarwal2019logarithmic} allow for control in LDS with adversarial loss functions. Provable control in the Gaussian noise setting via the policy gradient method was studied in \cite{fazel2018global}. These works operate in the absence of perturbations or assume that they are i.i.d. Gaussian, as opposed to adversarial which is what we consider. Other relevant work from the machine learning literature includes spectral filtering techniques for learning and open-loop control of partially observable systems \cite{hazan2017learning, arora2018towards, hazan2018spectral}.

\paragraph{Non-stochastic control:}
Regret minimization for control of dynamical systems with adversarial perturbations was initiated in the recent work of \cite{agarwal2019online}, who use online learning techniques and convex relaxation to obtain provable bounds for controlling LDS with adversarial perturbations. These techniques were extended in \cite{agarwal2019logarithmic} to obtain logarithmic regret under stochastic noise, in \cite{hazan2019nonstochastic} for the control of unknown systems, and in \cite{simchowitz2020improper} for control of systems with partially observed states.

\paragraph{System identification.}
For the stochastic setting, several works \cite{faradonbeh2018finite, simchowitz2018learning, sarkar2019near} propose to use the least-squares procedure for parameter identification. In the adversarial setting, least-squares can lead to inconsistent estimates. For the partially observed stochastic setting, \cite{oymak2019non, sarkar2019finite, simchowitz2018learning} give results guaranteeing parameter recovery using Gaussian inputs. Provable system identification in the adversarial setting was obtained in \cite{simchowitz2019learning,hazan2019nonstochastic}.

\section{Preliminaries} \label{sec:preliminaries}

\paragraph{Online convex optimization with memory.} \label{oco_mem}

The setting of online convex optimization (OCO) efficiently models iterative decision making. A player iteratively choses an action from a convex decision set $x_t \in \K \subseteq \reals^d$, and suffers a loss according to an adversarially chosen loss function $f_t(x_t)$. In the bandit setting of OCO, called Bandit Convex Optimization (BCO), the only information available to the learner after each iteration is the loss value itself, a scalar, and no other information about the loss function $f_t$. 

A variant which is relevant to our setting of control is BCO with memory. This is used to capture time dependence of the reactive environment. Here, the adversaries pick loss functions $f_t$ with bounded memory $H$ of our previous predictions, and as before we assume that we may observe the value but have no access to the gradient of our losses $f_t$. The goal is to minimize regret, defined as:
\begin{align*}
\mathop{\text{Regret}} = \mathop{\mathbb{E}}_{\mathcal{R}_\mathcal{A}}\left[\sum_{t=H}^T f_t(x_{t-\bar{H}:t})\right] - \min_{x^{\star} \in \mathcal{K}} \sum_{t=H}^T f_t(x^{\star}, \ldots, x^{\star}),
\end{align*}
where we denote $\bar{H}=H-1$ and $x_{t-\bar{H}:t} = (x_{t-\bar{H}}, \ldots, x_t)$ for clarity, $x_1, \ldots, x_T$ are the predictions of algorithm $\mathcal{A}$, and $\mathcal{R}_\mathcal{A}$ represents the randomness due to the algorithm $\mathcal{A}$.

For the settings of theorem \ref{thm:bco_main}, we assume that the loss functions $f_t$ are convex with respect to $x_{t-\bar{H}:t}$, $G$-Lipschitz, $\beta$-smooth, and bounded. We can assume without loss of generality that the loss functions are bounded by $1$ in order to simplify computations. In the case where the functions are bounded by some $|f_t(x_{t-\bar{H}:t})| \leq M$, one can obtain the same results with an additional factor $M$ in the regret bounds by dividing the gradient estimator by $M$.

\section{An Algorithm for BCO with Memory} \label{sec:alg}

This section describes the main building block for our control methods: an algorithm for BCO with memory. Our algorithm takes a non-increasing sequence of learning rates $\{ \eta_t\}_{t=1}^T$ and a \textit{perturbation constant} $\delta$, a hyperparameter associated with the gradient estimator. Note that the algorithm projects $x_t$ onto the Minkowski subset $\mathcal{K}_{\delta} = \{ x \in \mathcal{K} \: : \: \frac{1}{1 - \delta}x \in \mathcal{K} \}$ to ensure that $y_t = x_t + \delta u_t \in \mathcal{K}$ holds.

\begin{algorithm}
\caption{BCO with Memory}
\label{alg1}
\begin{algorithmic}[1]
    \STATE \textbf{Input:} $\mathcal{K}$, $T$, $H$, $\{ \eta_t \}$ and $\delta$
    \STATE Initialize $x_{1} = \cdots = x_{H} \in \mathcal{K}_{\delta}$ arbitrarily
    \STATE Sample $u_{1}, \ldots, u_{H} \in_{\mathbf{R}} \mathbb{S}_{1}^{d}$
    \STATE Set $y_i = x_i + \delta u_i$ for $i = 1, \ldots, H$
    \STATE Set $g_i = 0$ for $i = 1, \ldots, \bar{H}$
    \STATE Predict $y_i$ for $i = 1, \ldots, \bar{H}$
    \FOR {$t = H, \ldots, T$}
        \STATE predict $y_t$
        \STATE suffer loss $f_t(y_{t-\bar{H}:t})$ 
        \STATE store $g_t = \frac{d}{\delta}f_t(y_{t-\bar{H}:t}) \sum\limits_{i=0}^{\bar{H}} u_{t-i}$ \label{gt_def}
        \STATE set $x_{t+1} = \mathop{\Pi}\limits_{\mathcal{K}_{\delta}} \left[x_{t} - \eta_{t} \; g_{t-\bar{H}} \right]$
        \STATE sample $u_{t+1} \in_{\text{R}} \mathbb{S}_{1}^d$ \label{sample1}
        \STATE set $y_{t+1} = x_{t+1} + \delta u_{t+1}$
    \ENDFOR
    \STATE return
\end{algorithmic}
\end{algorithm}

The main performance guarantee for this algorithm is given in the following theorem:

\begin{theorem}\label{thm:bco_main}
Setting $\eta_t = \Theta(t^{-3/4} H^{-3/2} d^{-1} D^{2/3} G^{-2/3})$ and $\delta = \Theta(T^{-1/4} D^{1/3} G^{-1/3})$, Algorithm \ref{alg1} produces a sequence $\{y_t\}_{t=0}^{T}$ that satisfies:
\begin{align*}
\mathop{\text{Regret}} \leq \mathcal{O}\left(T^{3/4} H^{3/2} d D^{4/3} G^{2/3} \right).
\end{align*} 

In particular, $\text{\emph{Regret}} \leq \mathcal{O}\left(T^{3/4}\right)$.
\end{theorem}

%The goal is to reduce the problem from optimization on $f_t(x_{t-\bar{H}}, \ldots, x_t)$ to $f_t(x_t, \ldots, x_t)$, and then to perform gradient descent by showing that $g_t$ is a valid estimator of $\nabla f_t(x_t, \ldots, x_t)$.

The proof consists of four parts: In \ref{sec:sphere_estimator_basic} we cover notation for functions and sets relevant to our analysis. In \ref{sec:peturbation_properties}, we cover some properties of the exploration noises $u_t$. In \ref{sec:gradient_estimator}, we prove a few important lemmas about the gradient estimator $g_t$. Finally, in \ref{sec:proofBCO} we combine our lemmas from above with a reduction of the main theorem to obtain our main result.

\subsection{Notation and Basic Results}\label{sec:sphere_estimator_basic}

Denote the ball and sphere of dimension $d$ with radius $r$ respectively as 
\begin{align*}
\mathbb{B}_r^d \doteq \{ x \in \mathbb{R}^k \: : \: \norm{x} \leq r \} \ ,  \ 
\mathbb{S}_r^d \doteq \{ x \in \mathbb{R}^k \: : \: \norm{x} = r \}.
\end{align*}
Consider a convex set $\mathcal{K} \subset \mathbb{R}^d$ bounded with diameter $D$ and containing the unit ball $\mathbb{B}$.\footnote{We suppress the radius and dimensionality indices for $\mathbb{S}_1^d$ and $\mathbb{B}_1^d$ for the sake of presentation.} For $0 < \delta < 1$, consider the Minkowski subset: 
$$\mathcal{K}_{\delta} \doteq \{ x \in \mathcal{K} \: : \: \frac{1}{1 - \delta}x \in \mathcal{K} \},$$ 
and observe that $\mathcal{K}_{\delta}$ is convex and $\forall u \in B_1^d, x \in \mathcal{K}_{\delta}$ we have $x + \delta u \in \mathcal{K}$ because $\mathcal{K}$ contains the unit ball.

Next, we define the $\delta$-smoothed version of a function $f: \mathbb{R}^d \rightarrow \mathbb{R}$ to be:
\begin{align}\hat{f}_\delta(x) \doteq \mathop{\mathbb{E}} \limits_{v \sim \mathbb{B}} \left[ f(x + \delta v) \right] \label{eq:delta_smoothed_f}
\end{align}

The following standard facts about the gradient of a smoothed function can be found in the literature, e.g. \cite{hazan201210} chapter 2:

\begin{fact} \label{lem:basic_estimator} Let $f$ be $G$-Lipschitz, and $\hat{f}_\delta$ as defined in eq. \ref{eq:delta_smoothed_f}. We then have:
\begin{enumerate}
    \item $\mathop{\mathbb{E}}\limits_{u \sim \mathbb{S}} \left[ f(x + \delta u) u \right] = \dfrac{\delta}{d} \nabla \hat{f}_\delta (x)$
    \item $|\hat{f}_\delta (x) - f(x)| \leq \delta G, \; \forall x \in \mathcal{K}$
\end{enumerate}
\end{fact}

We additionally introduce the function $\tilde{f}_t : \mathcal{K} \rightarrow \mathbb{R}$ for loss functions with memory defined as:
\begin{align*} 
\tilde{f}_t(x) & \doteq f_t(\overbrace{x, \ldots, x}^{\times H})
\end{align*}

Throughout our analysis, it will be helpful to denote the collection of vectors $(v_{t-n}, \ldots, v_t)$ by $v_{t-n:t}$. Using this notation, addition and scalar multiplication will also be compactly expressed as $v_{t-n:t} + \alpha \, w_{t-n:t} \doteq (v_{t-n} + \alpha w_{t-n},\ldots, v_t + \alpha w_t)$. Because we are interested in loss functions with $H$ inputs, we will mostly be interested in collections of the form $v_{t-H+1:t}$. To avoid the excessive use of $H \pm 1$ throughout the rest of the paper, we will introduce the  notation $\bar{H} \doteq H - 1$. 

We now introduce the index-wise gradients $\nabla_i f_t$ to be the derivative of $f_t$ with respect to the $i$'th input vector, namely:
\begingroup
\setlength{\abovedisplayskip}{10pt}
\setlength{\belowdisplayskip}{10pt}
\begin{align*}
    \nabla_i f_t(x_{t-\bar{H}:t}) & = \frac{\partial f_t(x_{t-\bar{H}}, \ldots, x_t)}{\partial x_{t-\bar{H}+i}}
\end{align*}
\endgroup
such that $\nabla f_t = (\nabla_{0} f_t, \ldots, \nabla_{\bar{H}} f_t)$. 
We make the following observation about the gradients $\nabla_i f_t$.

\begin{lemma}\label{lem:chain_rule}
The gradient $\nabla \tilde{f}_t(x) = \frac{\partial \tilde{f}_t (x)}{\partial x}$ is related to the gradient of $f_t$ by
\begingroup
\setlength{\abovedisplayskip}{10pt}
\setlength{\belowdisplayskip}{10pt}
\begin{align*}
    \nabla \tilde{f}_t (x) = \sum_{i=0}^{\bar{H}} \nabla_i f_t (x_{t-\bar{H}}, \ldots, x_t) \biggr \vert_{x_{t-\bar{H}} = \ldots = x_t = x}
\end{align*}
\endgroup
which we denote as $\nabla \tilde{f}_t(x) = \sum\limits_{i=0}^{\bar{H}} \nabla_i \tilde{f}_t(x)$.
\end{lemma}

\begin{proof} Applying chain rule over $f_t(x_{t-\bar{H}:t})$ with $x_{t-i}(x) = x, \: i = 0, \ldots, \bar{H}$ yields the product of the $dH$ dimensional gradient $\frac{\partial f_t}{\partial x_{t-\bar{H}:t}}$ and the $dH \times d$ dimensional Jacobian $\frac{\partial x_{t-\bar{H}:t}}{\partial x}$, which is equal to $H$ copies of the $d \times d$ identity matrix. Specifically,
\begin{align*}
    \nabla \tilde{f}_t (x) & = \frac{\partial \tilde{f}_t (x)}{\partial x} 
     = \frac{\partial f_t (x_{t-\bar{H}:t})}{\partial x_{t-\bar{H}:t}}^{\top} \cdot \frac{\partial x_{t-\bar{H}:t}}{\partial x} \\
    & = \begin{bmatrix} \frac{\partial f_t(x_{t-\bar{H}:t})}{\partial x_{t-\bar{H}}} \\ \vdots \\ \frac{\partial f_t(x_{t-\bar{H}:t})}{\partial x_{t}} \end{bmatrix}^{\top} \cdot \begin{bmatrix} I_d \\ \vdots \\ I_d \end{bmatrix} \\
    & = \sum_{i=0}^{\bar{H}} \frac{\partial f_t(x_{t-\bar{H}:t})}{\partial x_{t-i}} 
     = \sum_{i=0}^{\bar{H}} \nabla_i f_t (x_{t-\bar{H}}, \ldots, x_t) \biggr \vert_{x_{t-\bar{H}} = \ldots = x_t = x}
\end{align*}
where the derivatives $\frac{\partial f_t(x_{t-\bar{H}:t})}{\partial x_{t-\bar{H}:t}}$ are evaluated at $x_{t-\bar{H}} = \ldots = x_t = x$ implicitly on lines 2 through 4 for clarity.
\end{proof}

Finally, we denote the optimizer over $\mathcal{K}$ with respect to all observed loss functions as $x^\star = \arg \min_{x \in \mathcal{K}} \sum_{t=H}^T f_t(x, \ldots, x)$, and its projection onto the corresponding Minkowski subset as $x^{\star}_{\delta} = \Pi_{\mathcal{K}_{\delta}}(x^{\star})$.

\subsection{Properties of the random exploration noise} \label{sec:peturbation_properties}

\begin{claim} \label{obs:independence}\normalfont{\textbf{(Independence)}}
$x_t$ is independent of $u_{t-\bar{H}}, \ldots, u_t$.
\end{claim} %make into claim
\begin{proof}
Base case: for $t \leq H$ all $x_t$'s are set arbitrarily to be equal and so the conclusion is immediate. For $t \geq H$: Assume this holds for $x_t$ and observe that $x_{t+1} = x_t + \eta_t g_{t-\bar{H}}$ is uniquely defined by $x_{t}$ and $g_{t-\bar{H}}$, for which the latter satisfies 
\begin{align*}
    g_{t-\bar{H}} = \frac{d}{\delta} f_{t-\bar{H}} (x_{t-2\bar{H}:t-\bar{H}}+ u_{t-2\bar{H}:t-\bar{H}}) \sum_{i = 0}^{H-1}u_{t-\bar{H}-i}.
\end{align*}
Now, since $f_{t-\bar{H}}$ and $u_{t-2\bar{H}:t-\bar{H}}$ are sampled before $u_{t-\bar{H}+1:t+1}$, the random variables that uniquely determine $g_t$ are independent from $u_{t-\bar{H}+1:t+1}$. Furthermore, by induction hypothesis $x_{t}$ is independent of $u_{t-\bar{H}}, \ldots, u_t$ and clearly also of $u_{t+1}$. Thus, the components that uniquely define $x_{t+1}$ are independent of $u_{t-\bar{H}+1:t+1}$, which means that $x_{t+1}$ is independent of $u_{t-\bar{H}+1:t+1}$ as well, as desired.
\end{proof}

\paragraph{Remark.} Claim \ref{obs:independence} above allows us to conclude that $u_{t-\bar{H}:t}$ is independent of $x_{t-\bar{H}:t}$, which crucially allows us to apply fact \ref{lem:basic_estimator} to our gradient estimator $g_t$.

\begin{lemma}\label{lem:perturbation_norm} The sum of $u_{t-\bar{H}}, \ldots, u_t$ for all $t$ has expected squared norm less than or equal to $H$.
\end{lemma}

\begin{proof}
Since $u_t \in_R \mathbb{S} \; \forall t$, we have $\mathbb{E}[u_i \cdot u_j] = 0$ whenever $i \neq j$, hence
\begin{align*}
    \mathbb{E} \left[ \norm{\sum_{i=0}^{\bar{H}} u_{t-i}}^2 \right] &= \mathbb{E} \left[ \left( \sum_{i=0}^{\bar{H}} u_{t-i} \right) \cdot \left( \sum_{i=0}^{\bar{H}} u_{t-i} \right) \right] \\
    & = \mathbb{E} \left[ \sum_{i=0}^{\bar{H}} \norm{ u_{t-i} }^2 \right] + \mathbb{E} \left[ \sum_{i\neq j} u_{t-i} \cdot u_{t-j} \right] \\
    & = H
\end{align*}
\end{proof}

\subsection{Properties of the gradient estimator}\label{sec:gradient_estimator}

\noindent
The goal of this section is to prove a lemma showing that our gradient estimator $g_t$ is a valid estimator of $\nabla \tilde{f}_{t}(x_{t+\bar{H}})$ by bounding the difference in expectation between the two, as well as bounding the norm of $g_t$ itself. We recall our previous assumptions that the loss functions are bounded by one, have gradients bounded by $||\nabla f_t || \leq G$ (which implies $f_t$ is $G$-Lipschitz), and have hessians bounded by $||\nabla^2 f_t || \leq \beta$ (which implies $f_t$ is $\beta$-smooth).

We start by bounding the expected square norm of our gradient estimator. We will use this to bound the distance between the predictions $x_t$ of our algorithm so that we may replace $\nabla f_t(x_{t-\bar{H}:t})$ with $\nabla \tilde{f}_t(x_t)$ in our analysis.

\begin{lemma}\label{lem:gt_norm}
The gradient estimator $g_t$ satisfies $\mathbb{E} \left[ \norm{g_t}^2 \right] \leq \frac{d^2 H}{\delta^2}$.
\end{lemma}

\begin{proof}
Combining lemma \ref{lem:perturbation_norm} with the assumption $f_t(y_{t-\bar{H}:t}) \leq 1$ and the definition of $g_t$, it follows that 
\begin{align}
    \mathbb{E} \left[ \norm{g_t}^2 \right] & = \mathbb{E}_{u_{t-\bar{H}:t}} \left[\norm{\frac{d}{\delta} f_t \left( x_{t-\bar{H}:t} + \delta u_{t-\bar{H}:t} \right) \cdot \sum_{i=0}^{\bar{H}} u_{t-i}}^2 \right] \notag\\
    & = \mathbb{E} \left[ \frac{d^2}{\delta^2} f_t(y_{t-\bar{H}:t})^2 \norm{\sum_{i=0}^{H-1} u_{t-i}}^2 \right] \notag \\
    & \leq \frac{d^2}{\delta^2} \mathbb{E} \left[ \norm{\sum_{i=0}^{H-1} u_{t-i}}^2 \right] \notag \\
    & \leq \frac{d^2 H}{\delta^2}. \notag
\end{align}
\noindent
\textbf{Remark:} \label{wlog_bdd_1} Even if the losses $f_t$ are bounded by some constant $M > 1$, the results for our algorithm and proof still hold if one scales down the gradient estimator to $\frac{1}{M}g_t$ and add a factor $M$ to the regret bound.
\end{proof}

Using the lemma above, we can now bound the distance between our predictions as follows:

\begin{lemma}\label{lem:sq_stability}
For $x_0, \ldots, x_T$ selected according to Algorithm \ref{alg1}, we have that:
\begin{align*}
\mathbb{E} \left[ \norm{x_{t-\bar{H}:t} - (x_{t+\bar{H}}, \ldots, x_{t+\bar{H}}) }^2 \right] & \leq 8 \eta_{t-\bar{H}}^2 \frac{d^2H^4}{\delta^2}, \\
\mathbb{E} \left[ \norm{x_{t-\bar{H}:t} - (x_{t}, \ldots, x_{t}) }^2 \right] & \leq \eta_{t-\bar{H}}^2 \frac{d^2H^4}{\delta^2}.
\end{align*}
\end{lemma}
\begin{proof}
Starting with the first inequality, since $x_{t+1} = \mathop{\Pi}\limits_{\mathcal{K}_{\delta}} [x_{t} - \eta_{t} g_{t-\bar{H}}]$, we have that:
\begin{align*}
\mathbb{E} \left[ \left\lVert (x_{t-\bar{H}}, \ldots, x_{t}) - (x_{t+\bar{H}}, \ldots, x_{t+\bar{H}}) \right\rVert^2 \right] & = \mathbb{E} \left[ \sum_{i=0}^{\bar{H}} \left\lVert x_{t+\bar{H}} - x_{t-i} \right\rVert^2 \right] \\
& \leq \mathbb{E} \left[ \sum_{i=1}^{2\bar{H}} \left( \sum_{j=1}^i \left\lVert x_{t+\bar{H}-j+1} - x_{t+\bar{H}-j} \right\rVert \right)^2 \right] & \text{($\bigtriangleup$-ineq.)}\\
& \leq \mathbb{E} \left[ \sum_{i=1}^{2\bar{H}} \left( \sum_{j=1}^i \eta_{t+\bar{H}-j} \left\lVert g_{t-j} \right\rVert \right)^2 \right] &\text{(projection property)}\\
& \leq \mathbb{E} \left[ \eta_{t-\bar{H}}^2 \sum_{i=1}^{2\bar{H}} \left( \sum_{j=1}^i \norm{g_{t-j}} \right)^2 \right] &\text{($\eta_t$ decreasing)}\\
& \leq 8 \eta_{t-\bar{H}}^2 H^3 \frac{d^2H}{\delta^2} & \text{(lemma \ref{lem:gt_norm})}
\end{align*}
and following the steps above but summing up to $\bar{H}$ instead of $2\bar{H}$, we similarly obtain the second inequality
\begin{align*}
    \mathbb{E} \left[ \norm{x_{t-\bar{H}:t} - (x_{t}, \ldots, x_{t}) }^2 \right] \leq \eta_{t-\bar{H}}^2 \frac{d^2H^4}{\delta^2}.
\end{align*}
\end{proof}

\begin{corollary} \label{cor:stability}
We also have that
\begin{align*}
    \mathbb{E} \left[ \norm{x_{t-\bar{H}:t} - (x_{t+\bar{H}}, \ldots, x_{t+\bar{H}}) } \right] \leq 3 \eta_{t-\bar{H}} \dfrac{d H^{2}}{\delta}, \\
    \mathbb{E} \left[ \norm{x_{t-\bar{H}:t} - (x_{t}, \ldots, x_{t}) } \right] \leq \eta_{t-\bar{H}} \dfrac{d H^{2}}{\delta}.
\end{align*}
\end{corollary}

\begin{proof}
This is an immediate consequence of lemma \ref{lem:sq_stability} since $\mathbb{E}[\norm{X}]^2 \leq \mathbb{E}[\norm{X}^2]$.
\end{proof}

We continue by proving our desired properties about the estimator $g_t$. We first observe the following property for \textit{linear} $\delta$-smoothed functions.

\begin{lemma}\label{lem:linear_estimator}
For $f$ \textbf{linear} and satisfying our assumptions, we have that
\begin{align*}
    \mathop{\E}\limits_{u_{t-\bar{H}:t} \sim \mathop{\oplus}\limits_{t = 1}^H \mathbb{S}} \left[ \dfrac{d}{\delta} f(x_{t-\bar{H}:t} + \delta u_{t-\bar{H}:t}) u_{t-\bar{H}:t} \right] & = \nabla f(x_{t-\bar{H}:t})
\end{align*}
\end{lemma}

\begin{proof}
By the independence of $x_{t-\bar{H}:t}$ and $u_{t-\bar{H}:t}$ (\ref{obs:independence}), we can apply Fact \ref{lem:basic_estimator} to each index $i = 0, \ldots, \bar{H}$ and obtain
\begin{align*}
\mathop{\E}\limits_{u_{t-\bar{H}:t} \sim \mathop{\oplus}\limits_{t = 1}^H \mathbb{S}} \left[ f(x_{t-\bar{H}:t} + \delta u_{t-\bar{H}:t}) u_{t-i} \right] &= \mathop{\E}\limits_{u_{t-i} \sim \mathbb{S}} \left[ \mathop{\E}\limits_{u_{[t-\bar{H}:t] \setminus \{t-i\}} \sim \mathop{\oplus}\limits_{t = 1}^{H-1} \mathbb{S}} \left[ f(x_{t-\bar{H}:t} + \delta u_{t-\bar{H}:t}) u_{t-i} \right] \right] \\
&= \mathop{\mathbb{E}} \limits_{u_{t-i} \sim \mathbb{S}} \left[ f(x_{t-\bar{H}:t} + \delta (\mathbf{0}, \ldots, u_{t-i}, \ldots, \mathbf{0})) u_{t-i} \right] \\
&= \dfrac{\delta}{d} \nabla_{\bar{H}-i} \hat{f}_{\delta}(x_{t-\bar{H}:t}) \\
&= \dfrac{\delta}{d} \nabla_{\bar{H}-i} f(x_{t-\bar{H}:t})
\end{align*}
where the second and last lines follows by the linearity of $f$ and the symmetry of the sphere.
Since $\nabla f(x_{t-\bar{H}:t}) = (\nabla_0 f(x_{t-\bar{H}:t}), \ldots, \nabla_{\bar{H}} f(x_{t-\bar{H}:t}))$, the lemma then follows.
\end{proof}

Using the theorem above, we can generalize fact \ref{lem:basic_estimator} in the following manner:

\begin{theorem} \label{lem:prod_basic_estimator} For general convex $f$ satisfying our assumptions, we have:
\begin{align*}
    \norm{ \mathop{\E} \limits_{u_{t-\bar{H}:t} \sim \mathop{\oplus}\limits_{t = 1}^H \S} \left[ \dfrac{d}{\delta} f(x_{t-\bar{H}:t} + \delta u_{t-\bar{H}:t}) u_{t-\bar{H}:t} \right] - \nabla f(x_{t-\bar{H}:t}) } \leq 2 d\delta GH
\end{align*}
\end{theorem}

\begin{proof}
Consider the linear function $\bar{f}_{x_{t-\bar{H}:t}}(z_{t-\bar{H}:t}) = f(x_{t-\bar{H}:t}) + \nabla f(x_{t-\bar{H}:t}) (z_{t-\bar{H}:t}-x_{t-\bar{H}:t})$. By lemma \ref{lem:linear_estimator} above,
\begin{align*}
    \mathop{\E}\limits_{u_{t-\bar{H}:t} \sim \mathop{\oplus}\limits_{t = 1}^H \mathbb{S}} \left[ \frac{d}{\delta} \bar{f}_{x_{t-\bar{H}:t}}(x_{t-\bar{H}:t} + \delta u_{t-\bar{H}:t})u_{t-\bar{H}:t} \right] &= \nabla \bar{f}_{x_{t-\bar{H}:t}} (x_{t-\bar{H}:t}) \\
    &= \nabla f (x_{t-\bar{H}:t}).
\end{align*}
The lemma then follows when we bound the difference between ${f}$ and $\bar{f}_{x_{t-\bar{H}:t}}$ such that:
    \begin{align*}
    & \quad \norm{\mathop{\E} \left[ \frac{d}{\delta} f (x_{t-\bar{H}:t} + \delta u_{t-\bar{H}:t})u_{t-\bar{H}:t} \right] - \nabla f(x_{t-\bar{H}:t}) } & \\
    &\leq \norm{ \mathop{\E} \left[ \frac{d}{\delta} f(x_{t-\bar{H}:t} + \delta u_{t-\bar{H}:t})u_{t-\bar{H}:t} \right] - \mathop{\E} \left[ \frac{d}{\delta} \bar{f}_{x_{t-\bar{H}:t}}(x_{t-\bar{H}:t} + \delta u_{t-\bar{H}:t})u_{t-\bar{H}:t} \right] } & \\
    &\leq \mathop{\E} \left[ \frac{d}{\delta} | f(x_{t-\bar{H}:t} + \delta u_{t-\bar{H}:t}) - \bar{f}_{x_{t-\bar{H}:t}}(x_{t-\bar{H}:t} + \delta u_{t-\bar{H}:t})| \norm{\delta u_{t-\bar{H}:t}} \right] & \\
    & \leq \mathop{\E} \left[ \frac{d}{\delta} |f(x_{t-\bar{H}:t} + \delta u_{t-\bar{H}:t}) - \left(f(x_{t-\bar{H}:t}) + \nabla f(x_{t-\bar{H}:t}) \cdot \delta u_{t-\bar{H}:t} \right)| \norm{\delta u_{t-\bar{H}:t}} \right] & \\
    & \leq \mathop{\E} \left[ \frac{2d}{\delta} G\norm{\delta u_{t-\bar{H}:t}}^2 \right] & (\text{$G$-Lipschitzness}) \\
    &\leq 2d\delta GH & 
\end{align*}
where the expectations are taken over $u_{t-\bar{H}:t} \sim \mathop{\oplus}\limits_{t = 1}^H \mathbb{S}$.
\end{proof}

\begin{corollary} \label{lem:unbiased_estimator}
$g_t$ satisfies:
\begin{align*}
    \norm{\mathbb{E}[g_t] - \sum_{i=0}^{\bar{H}} \nabla_i  f_{t} (x_{t-\bar{H}:t})} \leq 2d\delta H^{3/2} G
\end{align*}
where the expectation is over all randomness in the algorithm.
\end{corollary}

\begin{proof} 
By the definition of $g_t$ (line \ref{gt_def} of Algorithm \ref{alg1}), we have \begin{align*}
    & \quad \norm{\mathbb{E}[g_t] - \sum_{i=0}^{\bar{H}} \nabla_i  f_{t} (x_{t-\bar{H}:t})} & \\ &\leq \norm{\mathop{\E} \left[ \frac{d}{\delta} f (x_{t-\bar{H}:t} + \delta u_{t-\bar{H}:t}) \sum_{i=0}^{\bar{H}} u_{t-i} \right] - \sum_{i=0}^{\bar{H}} \nabla_i  f_{t} (x_{t-\bar{H}:t})} & \\
    &\leq \sqrt{H} \norm{ \mathop{\E} \left[ \dfrac{d}{\delta} f(x_{t-\bar{H}:t} + \delta u_{t-\bar{H}:t}) u_{t-\bar{H}:t} \right] - \nabla f(x_{t-\bar{H}:t})} & (\text{Cauchy-Schwarz}) \\
    &\leq 2d \delta GH^{3/2}. & (\text{lemma \ref{lem:prod_basic_estimator}})
\end{align*}
\end{proof}

\begin{lemma}\label{lem:gradient_difference}
We have that:
\begin{align*}
\mathbb{E}\left[\norm{ \sum_{i=0}^{\bar{H}} \nabla_i  f_{t} (x_{t-\bar{H}:t}) - \nabla \tilde{f}_{t} (x_{t+\bar{H}})}^2 \right] \leq 3 \frac{\beta\eta_{t-\bar{H}}H^{5/2}d}{\delta}
\end{align*}
\end{lemma}
\begin{proof}
Using the results derived thus far, we obtain:
\begin{align*}
    & \quad \mathbb{E}\left[\norm{ \sum_{i=0}^{\bar{H}} \nabla_i  f_{t} (x_{t-\bar{H}:t}) - \nabla \tilde{f}_{t} (x_{t+\bar{H}})}^2 \right] & \\
    &= \mathbb{E} \left[ \norm{\sum_{i=0}^{\bar{H}} \nabla_i f_{t}(x_{t-\bar{H}:t}) - \sum_{i=0}^{\bar{H}} \nabla_i f_{t}(x_{t+\bar{H}}, \ldots, x_{t+\bar{H}})}^2 \right] & \\
    & \leq H \mathbb{E} \left[ \sum_{i=0}^{\bar{H}} \norm{\nabla_i f_{t}(x_{t-\bar{H}:t}) - \nabla_i f_{t}(x_{t+\bar{H}}, \ldots, x_{t+\bar{H}})}^2 \right] & \text{(Cauchy-Schwarz)} \\
    & = H  \mathbb{E} \left[ \norm{\nabla f_{t}(x_{t-\bar{H}}, \ldots, x_t) - \nabla f_{t}(x_{t+\bar{H}}, \ldots, x_{t+\bar{H}})}^2 \right] & \text{(\ref{lem:chain_rule})} \notag \\
    & \leq H \beta^2 \mathbb{E} \left[ \norm{(x_{t-\bar{H}}, \ldots,  x_t) - (x_{t+\bar{H}}, \ldots, x_{t+\bar{H}})}^2 \right] \notag \\
    & \leq 8 H\beta^2  \frac{\eta_{t-\bar{H}}^2H^{4}d^2}{\delta^2} & \text{(\ref{lem:sq_stability})} \\
    & = \frac{8 \beta^2\eta_{t-\bar{H}}^2H^{5}d^{2}}{\delta^2}.
\end{align*}
after which our result follows by $\mathbb{E}[\norm{X}]^2 \leq \mathbb{E}[\norm{X}^2]$.
\end{proof}

The lemmas above allow us to obtain our desired result regarding the gradient estimator $g_t$, presented below.

\begin{corollary}\label{lemma2}
The gradient estimator $g_t$ satisfies:
\begin{align*}
\mathbb{E} \left[ \norm{\mathbb{E}[g_t] - \nabla \tilde{{f}}_{t}(x_{t+\bar{H}})} \right] 
&\leq 2 d \delta H^{3/2} G + 3 \frac{\beta\eta_{t-\bar{H}}H^{5/2}d}{\delta}
\end{align*}
\end{corollary}

\begin{proof}
This follows from Corollary \ref{lem:unbiased_estimator} and Lemma \ref{lem:gradient_difference} due to the triangle inequality.
\end{proof}

\subsection{Proof of Theorem \ref{thm:bco_main}} \label{sec:proofBCO}

We start by performing a reduction from bounding the regret over $f_t(y_{t-\bar{H}:t}) - \tilde{f}_t(x_t)$ to that over $\tilde{f}_t(x^{\star}) - \tilde{f}_t(x^{\star}_{\delta})$.

\begin{lemma} \label{lem:dist_to_policy}
We have that:
\begin{align*}
    \E\left[ \sum_{t=H}^T \left( f_t(y_{t-\bar{H}:t}) - \tilde{f}_{t} (x^{\star}) \right)\right] - \E\left[ \sum_{t=H}^T \left( \tilde{f}_{t} (x_{t}) - \tilde{f}_{t} (x^{\star}_{\delta}) \right)\right] \leq 2\delta G D \sqrt{H} T + \dfrac{d G H^{2}}{\delta} \sum_{t=1}^T \eta_t
\end{align*}
\end{lemma}
\begin{proof} Using that $f_t$ is $G$-Lipschitz, we have that:
\begin{align*}
\E\left[ \left( f_t(y_{t-\bar{H}:t}) - \tilde{{f}}_{t} (x_{t}) \right)\right] &= \E\left[ \left( f_t(x_{t-\bar{H}:t} + \delta u_{t-\bar{H}:t}) - \tilde{f}_{t} (x_{t}) \right)\right]  \\
&\leq \E\left[ \left( f_t(x_{t-\bar{H}:t}) - \tilde{{f}}_{t} (x_{t}) \right)\right] + \delta G \sqrt{H} \\
&\leq \dfrac{d G H^{2}\eta_{t-\bar{H}}}{\delta}  + \delta G \sqrt{H} & \text{(\ref{cor:stability})}
\end{align*}
and by the properties of $\mathcal{K}_\delta$, we have that $\left\lVert x^{\star} - x^{\star}_{\delta} \right\rVert \leq \delta D$ and therefore
\begin{align*}
|\tilde{f}_t(x_\delta^{\star}) - \tilde{f}_t(x^{\star})| &\leq G \left\lVert (x^{\star}, \ldots, x^{\star}) - (x^{\star}_{\delta}, \ldots, x^{\star}_{\delta}) \right\rVert \\
&\leq \delta GD \sqrt{H}.
\end{align*}
Combining these two results concludes the proof.
\end{proof}

We now move on to bounding the regret over $\tilde{f}_t(x_t) - \tilde{f}_t(x^{\star}_{\delta})$.

\begin{observation} \label{obs:total_exp}
If we denote by $\mathbb{E}$ the expectation over the $u_t$'s and apply the law of total expectation, we have that:
\begin{align*}
\mathbb{E}\left[ (\mathbb{E}[g_{t-\bar{H}}] - g_{t-\bar{H}}) \cdot (x_{t} - x_{\delta}^*)\right] = \mathbb{E}\left[ \mathbb{E}[(\mathbb{E}[g_{t-\bar{H}}] - g_{t-\bar{H}}) \cdot (x_{t} - x_{\delta}^*) \; | \; (u_0, \ldots, u_{t-\bar{H}})]\right] = 0
\end{align*}
\end{observation}

\begin{observation} \label{obs:conv} By convexity of $\tilde{{f}}_{t}$, we have that: 
\begin{align*}
    \tilde{{f}}_{t}(x_{t}) - \tilde{{f}}_{t}(x^{\star}_{\delta}) \leq \nabla \tilde{f}_{t}(x_{t})^\top (x_{t} - x^{\star}_{\delta})
\end{align*}
\end{observation}

\begin{lemma} \label{lem:regret_f_hat_tilde}
The delayed regret in terms of $\tilde{{f}}$ satisfies:
\begin{align*}
    \mathbb{E} \left[ \sum_{t = H}^{T} \tilde{{f}}_{t}(x_{t}) \right] - \sum_{t = H}^{T} \tilde{f}_t(x^{\star}_{\delta}) \leq \frac{D^2}{\eta_{T}} + \frac{d^2 H}{2\delta^2} \cdot \sum_{t = 1}^{T} \eta_t + 3 \frac{\beta H^{3}dD}{\delta} \sum_{t = 1}^{T} \eta_{t} + 2 d \delta H^{3/2} GDT
\end{align*}
\end{lemma}

\begin{proof}
First, observe that:
\begin{align}
    \left\lVert x_{t+1} - x^{\star}_{\delta} \right\rVert^2 & = \left\lVert \Pi_{\mathcal{K}_{\delta}}[x_{t} - \eta_{t} g_{t-\bar{H}}] - x^{\star}_{\delta} \right\rVert^2 \notag \\
    & \leq \left\lVert x_{t} - \eta_{t} g_{t- \bar{H}} - x^{\star}_{\delta} \right\rVert^2 \notag \\
    & = \left\lVert x_{t} - x^{\star}_{\delta} \right\rVert^2 + \left\lVert \eta_{t} g_{t-\bar{H}} \right\rVert^2 - 2 \eta_{t} g_{t-\bar{H}}^\top \cdot (x_{t} - x^{\star}_{\delta}) \notag \\
    \Rightarrow \quad \quad 2 g_{t-\bar{H}}^\top \cdot (x_{t} - x^{\star}_{\delta}) & \leq \frac{\left\lVert x_{t} - x^{\star}_{\delta} \right\rVert^2 - \left\lVert x_{t+1} - x^{\star}_{\delta} \right\rVert^2}{\eta_{t}} + \eta_{t} \left\lVert g_{t-\bar{H}} \right\rVert^2. \label{eq:it_ineq}
\end{align}
Therefore, we get:
\begin{align*}
    \mathbb{E} \left[ \sum_{t = H}^{T} \tilde{{f}}_{t}(x_{t}) \right] - \sum_{t = H}^{T} \tilde{{f}}_{t}(x^{\star}_{\delta})
    & = \mathbb{E} \left[ \sum_{t = H}^{T} \left(\tilde{{f}}_{t}(x_{t}) - \tilde{{f}}_{t}(x^{\star}_{\delta}) \right) \right] \\
    & \leq \mathbb{E} \left[ \sum_{t = H }^{T} \nabla \tilde{f}_{t}(x_{t})^\top \left(x_{t} - x^{\star}_{\delta} \right) \right] \\
    & = \mathbb{E} \left[ \sum_{t = H}^{T} \left(g_{t-\bar{H}} + (\E[g_{t-\bar{H}}] - g_{t-\bar{H}}) + (\nabla \tilde{f}_{t} (x_{t}) - \E[g_{t-\bar{H}}])\right)^\top \left(x_{t} - x^{\star}_{\delta} \right) \right]
\end{align*}
By equation \ref{eq:it_ineq}, observation \ref{obs:total_exp} and Cauchy-Schwarz, we have:
\begin{align*}
    \leq & \frac{1}{2} \mathbb{E} \left[ \sum_{t = H}^{T} \left(\frac{\left\lVert x_{t} - x^{\star}_{\delta}\right\rVert^2 - \left\lVert x_{t+1} - x^{\star}_{\delta} \right\rVert^2}{\eta_{t}} + \eta_{t} \left\lVert g_{t-\bar{H}} \right\rVert^2 \right) \right] + 0 \\
    & + \mathbb{E} \left[ \sum_{t = H }^{T} \left\lVert \nabla \tilde{{f}}_{t}(x_{t}) - \mathbb{E}[g_{t-\bar{H}}] \right\rVert \cdot \left\lVert x_{t} - x^{\star}_{\delta} \right\rVert \right] \\
    \leq & \frac{1}{2} \mathbb{E} \left[ \sum_{t = H}^{T} \left\lVert x_{t} - x^{\star}_{\delta} \right\rVert^2 \left(\frac{1}{\eta_{t}} - \frac{1}{\eta_{t-1}}  \right) + \frac{||x_{H} - x^{\star}_{\delta}||^2}{\eta_{\bar{H}}} \right] + \frac{d^2 H}{2\delta^2} \cdot \sum_{t = H}^{T} \eta_t & \text{(\ref{lem:gt_norm})} \\
    & + \sum_{t=H}^T \left( 2 d \delta H^{3/2} G + 3 \frac{\beta \eta_{t-\bar{H}} H^{5/2}d}{\delta} \right) \cdot D & \text{(\ref{lemma2})}
\end{align*}
Observing that $\eta_t$ is decreasing, we have:
\begin{align*}
    \mathbb{E} \left[ \sum_{t = H}^{T} \tilde{{f}}_{t}(x_{t}) \right] - \sum_{t = H}^{T} \tilde{f}_{t}(x^{\star}_{\delta}) & \leq D^2 \left( \frac{1}{2\eta_{\bar{H}}} + \frac{1}{2\eta_{T}} \right) + \frac{d^2 H}{2\delta^2} \cdot \sum_{t = 1}^{T} \eta_t + 3 \frac{\beta H^{3}dD}{\delta} \sum_{t = 1}^{T} \eta_{t} + 2 d \delta H^{3/2} GDT  \\
    & \leq \frac{D^2}{\eta_{T}} + \frac{d^2 H}{2\delta^2} \cdot \sum_{t = 1}^{T} \eta_t + 3 \frac{\beta H^{3}dD}{\delta} \sum_{t = 1}^{T} \eta_{t} + 2 d \delta H^{3/2} GDT 
\end{align*}
\end{proof}

We are now able to conclude our main proof.
\paragraph{Theorem \ref{thm:bco_main}.}
Setting $\eta_t = \Theta(t^{-3/4} H^{-3/2} d^{-1} D^{2/3} G^{-2/3})$ and $\delta = \Theta(T^{-1/4} D^{1/3} G^{-1/3})$, Algorithm \ref{alg1} produces a sequence $\{y_t\}_{t=0}^{T}$ that satisfies:
\begin{align*}
\mathop{\text{Regret}} \leq \mathcal{O}\left(T^{3/4} H^{3/2} d D^{4/3} G^{2/3} \right)
\end{align*} 

\begin{proof}
Putting \ref{lem:dist_to_policy} and \ref{lem:regret_f_hat_tilde} together, we get: 
\begin{align*}
\mathop{\text{Regret}} &= \mathbb{E} \left[  \sum_{t=H}^T \left( f_t(y_{t-\bar{H}:t}) - \tilde{f}_t(x^{\star}) \right) \right] \\
&\leq \left( \frac{D^2}{\eta_{T}} + \frac{d^2 H}{2\delta^2} \cdot \sum_{t = 1}^{T} \eta_t + 3 \frac{\beta H^{3}dD}{\delta} \sum_{t = 1}^{T} \eta_{t} + 2 d \delta H^{3/2} GDT \right) + \left( 2\delta G D \sqrt{H} T + \dfrac{d G H^{2}}{\delta} \sum_{t=1}^T \eta_t \right) \\
&\leq \frac{D^2}{\eta_{T}} + \frac{d^2 H}{2\delta^2} \cdot \sum_{t = 1}^{T} \eta_t + 4 \frac{\beta H^{3}dGD}{\delta} \sum_{t = 1}^{T} \eta_{t} + 4 d \delta H^{3/2} GDT 
\end{align*}
Setting the parameters as specified yields the desired result, concluding the proof of theorem \ref{thm:bco_main}.
\end{proof}

\section{Application to Online Control of LDS}\label{sec:control}

In this section we introduce the setting of online bandit control and relevant assumptions, with the objective of converting our control problem to one of BCO with memory, which will allow us to use algorithm \ref{alg1} to control linear dynamical systems using only bandit feedback.

In online control, the learner iteratively observes a state $x_t$, chooses an action $u_t$, and then suffers a convex cost $c_t(x_t, u_t)$ selected by an adversary. We assume for simplicity of analysis that $x_0=0$. Since the adversary can set $w_0$ arbitrarily, this does not change the generality of this setting. Because we are working in the bandit setting, we may only observe the value of $c_t(x_t, u_t)$ and have no access to the function $c_t$ itself. Therefore, the learner cannot apply $c_t$ to a different set of inputs, nor take gradients over them. As such, previous approaches to non-stochastic control such as in \cite{agarwal2019online,hazan2019nonstochastic} are no longer viable, as these rely on the learner being capable of executing both of these operations.

\paragraph{Assumptions.} From hereon out, we assume that the perturbations are bounded, i.e. $\norm{w_t} \leq W$, and that all $x_t$'s and $u_t$'s are bounded such that $\norm{x_t}, \norm{u_t} \leq D$. We additionally bound the norm of the dynamics $\norm{A} \leq \kappa_A, \norm{B} \leq \kappa_B$, and assume that the cost functions $c_t$ are Lipschitz and $\beta$-smooth.

\paragraph{Comparator class.} As in the existing literature, we measure our performance against the class of disturbance action controllers.

\begin{definition}\textbf{(Disturbance Action Controller)}
A disturbance action controller is parametrized by a sequence of $H$ matrices $M = [M^{[i]}]_{i=1}^{H}$ and a stabilizing $K$, acting according to $u_t = -K x_t + \sum\limits_{s=1}^{H} M^{[s]} w_{t-s}$.
\end{definition}

\noindent\textbf{DAC Policy Class.}
We define $\mathcal{M}$ to be the set of all disturbance action controllers (for a fixed $H$ and $K$) with geometrically decreasing component norms, i.e. $\mathcal{M} \doteq \{M \text{ s.t. } ||M^{[i]}|| \leq \kappa^3 \kappa_B(1-\gamma)^i\}$.

\paragraph{Performance metric.} For algorithm $\mathcal{A}$ that goes through the states $x_0, \ldots, x_T$, selects actions $u_0, \ldots, u_T$, and observes the sequence of perturbations $w=(w_0, \ldots, w_T)$, we define the expected total cost over any randomness in the algorithm given the observed disturbances to be
\begin{align*}
    J_T\left(\mathcal{A} | w \right) & = \mathop{\mathbb{E}}\limits_{\mathcal{A}} \left[\sum_{t=0}^T c_t(x_t, u_t)\right].
\end{align*}

With some slight abuse of notation, we will use $J_T\left(M | w \right)$ to denote the cost of the fixed DAC policy that chooses $u_t = -Kx_t + \sum\limits_{s=1}^{H} M^{[s]} w_{t-s}$ and observes the same perturbation sequence $w$. Following the literature on non-stochastic control, our metric of performance is regret, which for algorithm $\A$ is defined as
\begin{align}
\mathop{\text{Regret}} = \sup_{w_{1:T}}\left[  J_T\left(\mathcal{A}|w\right) - \min_{M \in \mathcal{M}}  \left[J_T\left(M|w\right)\right] \right]. \label{regret_form_2}
\end{align}

\section{Non-stochastic control of known systems}\label{sec:known}

We now give an algorithm for controlling known time-invariant linear dynamical systems in the bandit setting. Our approach is to design a disturbance action controller and to train it using our algorithm for BCO with memory. Formally, at time $t$ we choose the action $u_t = -Kx_t + \sum_{i=1}^H M_t^{[i]} w_{t-i} \label{controller_dyn}$ where $M_t = \{M_t^{[1]}, \ldots, M_t^{[H]}\} \in \mathbb{R}^{H \times m \times n}$ are the learnable parameters and we denote $w_t = 0, \; \forall t < 0$, for convenience. Note that $K$ does not update over time, and only exists to make sure that the system remains stable under the initial policy. \\

In order to train these controllers in the bandit setting, we identify the costs $c_t(x_t, u_t)$ with a loss function with memory that takes as input the past $H$ controllers $M_{t-\bar{H}}, \ldots, M_t$, and apply our results from algorithm \ref{alg1}. We denote the corresponding Minkowski subset of $\mathcal{M}$ by $\mathcal{M}_{\delta}$. Our algorithm is given below, and the main performance guarantee for it is given in the Theorem \ref{thm:control-known}, along with its proof.

\begin{algorithm}
\caption{Bandit Perturbation Controller}
\label{alg2}
\begin{algorithmic}[1]
    \STATE \textbf{Input:} $K, H, T, \{ \eta_t \}, \delta$, and $\mathcal{M}$
    \STATE Initialize $M_{0} = \cdots = M_{\bar{H}} \in \mathcal{M}_{\delta}$ arbitrarily
    \STATE Sample $\epsilon_{0}, \ldots, \epsilon_{\bar{H}} \in_{\mathbb{R}} S_1^{H \times m \times n}$
    \STATE Set $\widetilde{M}_i = M_i + \epsilon_i$ for $i = 0, \ldots, \bar{H}$
    \STATE Set $g_i = 0$ for $i = -\bar{H}, \ldots, 0, \ldots, \bar{H}$
    \FOR {$t = 0, \ldots, T$} 
        \STATE choose action $u_t = -K x_t + \sum_{i=1}^H \widetilde{M}_t^{[i]} w_{t-i}$
        \STATE suffer loss $c_t(x_{t}, u_{t})$
        \STATE observe new state $x_{t+1}$
        \STATE record $w_t = x_{t+1} - A x_t - B u_t$
        \STATE store $g_t = \dfrac{mnH}{\delta} c_t(x_{t}, u_{t}) \sum\limits_{i=0}^{\bar{H}} \epsilon_{t-i} \:$ if $t \geq H$ else 0
        \STATE set $M_{t+1} = \Pi_{\mathcal{M}_{\delta}} \left[M_{t} - \eta_{t} \, g_{t-\bar{H}} \right]$
        \STATE sample $\epsilon_{t+1} \in_\text{R} S_{1}^{H \times m \times n}$
        \STATE set $\widetilde{M}_{t+1} = M_{t+1} + \delta \epsilon_{t+1}$
        \ENDFOR
    \STATE return
\end{algorithmic}
\end{algorithm}

\begin{theorem}\label{thm:control-known}
If we set $\eta_t$ and $\delta$ as in theorem $\ref{thm:bco_main}$ and $H = \Theta\left(\log{T}\right)$, the regret incurred by Algorithm \ref{alg2} satisfies
\begin{align*}
    \text{\emph{Regret}} \leq \mathcal{O}\left(T^{3/4}\log^{5/2}T\right).
\end{align*}
\end{theorem}

\begin{proof}
Observe that, if we fix $x_{t-\bar{H}}$ (the state starting $\bar{H}$ time steps back) and the observed disturbances $w_{t-2\bar{H}-1}, \ldots, w_t$, then the state $x_t$ and action $u_t$ at $\bar{H}$ time steps later are uniquely determined by the sequence of $H$ policies $M_{t-\bar{H}}, \ldots M_t$, which means that $c_t(x_t, u_t)$ can be considered as an implicit functions of the past $H$ policies played. It then follows that $\forall c_t$, $\exists$ unique $f_t$ such that
\begin{align*}
    f_t(M_{t-\bar{H}}, \ldots M_t) \equiv c_t \left( x_t (M_{t-\bar{H}:t}), u_t (M_{t-\bar{H}:t}) | x_{t-H}, w_{t-2\bar{H}-1:t}\right).
\end{align*}

\noindent
Due to the analysis by \cite{agarwal2019online}, sections 4.3 and 4.4, we know that $f_t$ is convex with respect to $M_{t-\bar{H}}, \ldots, M_t$ when $x_{t-\bar{H}}$, $K$, and the perturbations $w_t$ are fixed. Furthermore, because $c_t$ is Lipschitz and smooth, $f_t$ is Lipschitz and smooth as well. This means we can successfully apply the approach in algorithm \ref{alg1} to our current setting. Therefore, by theorem \ref{thm:bco_main} we get that for any fixed initial $(\kappa, \gamma)$-stable $K$, if we denote the actions taken by Algorithm 2 as $u^K_0, \ldots, u^K_T$ , and $M^* = \arg\min\limits_{M \in \mathcal{M}} \sum_{t=H}^T c_t(x_t^K(M), u_t^K(M))$ the best DAC policy in hindsight, then
\begin{align*}
    & \quad \mathbb{E} \left[\sum_{t=0}^T c_t(x_t^K, u_t^K)\right] -  \sum_{t=0}^T c_t(x_t^K(M^{\star}), u_t^K(M^{\star})) \\
    & \leq \frac{D^2}{\eta_{T}} + \frac{d^2 H}{2\delta^2} \cdot \sum_{t = 1}^{T} \eta_t + 4 \frac{\beta H^{3}dGD}{\delta} \sum_{t = 1}^{T} \eta_{t} + 4 d \delta H^{3/2} GDT 
\end{align*}
\noindent
where $d=Hmn$ because each policy $M_{t}$ consists of $H$ matrices of dimension $m \times n$. Setting $H = \Theta{(\log T)}$ and the other parameters as in \ref{thm:bco_main}, we get $J_T(BPC) -  J_T(M^*) \leq \O(T^{3/4}\log^{5/2}T)$, where the factor $\log^{5/2}T$ follows from $d = \Theta{(H)}$ and $H = \Theta{(\log T)}$.

\end{proof}

\section{Non-stochastic control of unknown systems}\label{sec:unknown}

We now extend our algorithm to unknown systems, yielding a controller that achieves \textbf{sublinear regret for both unknown costs and unknown dynamics} in the non-stochastic adversarial setting. The main challenge in this scenario is that we are competing with the best linear policy $K\in \mathcal{K}$ that has access to the true dynamics. Moreover, if we don't know the the system, we are also unable to deduce the true $w_t$'s. While this initially may appear to be specially problematic for the class of perturbation-based controllers, we show that it is still possible to attain sublinear regret.

\subsection{System identification via method of moments}

Our approach to control of unknown systems follows the explore-then-commit paradigm, identifying the underlying dynamics up to some desireble accuracy using random inputs in the exploration phase, followed by running algorithm \ref{alg2} on the estimated dynamics. The approximate system dynamics allow us to obtain estimates of the perturbations, thus facilitating the execution of the perturbation-based controller. The procedure used to estimate the system dynamics is given in algorithm \ref{alg:sys-id}. 

One essential property we need is strong controllability, as defined by \cite{cohen2018online}. Controllability for a linear system is characterized by the ability to drive the system to any desired state through appropriate control inputs in the presence of deterministic dynamics, i.e. when the perturbations $w_t$ are 0.

\begin{definition}\label{sec:strongly-controllable}
A linear dynamical system with dynamics matrices $A, B$ is \textit{controllable} with controllability index $k \geq 1$ if the matrix
\begin{align*}
    C_k = \left[ B, AB, \ldots, A^{k-1}B \right] \enspace \in \mathbb{R}^{n \times km}
\end{align*}
has full row-rank. In addition, such as system is also considered $(k, \kappa)$-\textit{strongly controllable} if $\norm{(C_k C_k^{\top})^{-1}} \leq \kappa$.
\end{definition}

In order to prove regret bounds in the setting of unknown systems, we must ensure that the system remains somewhat controllable during the exploration phase, which we do by introducing the following assumptions which are slightly stronger than the ones required in the known system setting:

\begin{assumption}
We assume that the perturbation sequence is chosen at the start of the interaction, implying that this sequence $w_t$ does not depend on the choice of $u_t$.
\end{assumption}

\begin{assumption}
The learner knows a linear controller $\mathbb{K}$ that is $(\kappa, \gamma)$-strongly stable for the true, but unknown, transition matrices $(A,B)$ defining the dynamical system.
\end{assumption} 

\begin{assumption}
The linear dynamical system $(A - B\mathbb{K}, B)$ is $(k, \kappa)$-strongly controllable.
\end{assumption}

Note then that $\mathbb{K}$ is any stabilizing controller ensuring that the system remains controllable under the random actions, and $k$ the controllability index of the system.

\begin{algorithm}
\caption{System identification via random inputs}
\label{alg:sys-id}
\begin{algorithmic}[1]
    \STATE \textbf{Input:} $T_0, \mathbb{K}$
    \FOR {$t = 0, \ldots, T_0$}
        \STATE sample $\xi_t \in_\text{R} \{\pm 1\}^m$
        \STATE choose action $u_t = -\mathbb{K} x_t + \xi_t$
        \STATE Incur loss $c_t(x_t, u_t)$,  record $x_t$
    \ENDFOR
    \STATE set $N_j = \frac{1}{T_0 - k}\sum\limits_{t=0}^{T_0 - k - 1} x_{t+j+1} \xi_t^T$ for all $j$ in $[k]$
    \STATE Let $C_0 = (N_0, \ldots N_{k-1})$,  $C_1 = (N_1, ... N_k)$
    \STATE set $\hat{A} = C_1 C_0^T (C_0 C_0^T)^{-1} + N_0 K$ and $\hat{B} = N_0$
\end{algorithmic}
\end{algorithm}

\subsection{The algorithm and regret guarantee}

Combining algorithm \ref{alg2} with the system identification method in algorithm \ref{alg:sys-id}, we obtain the following algorithm for the control of unknown systems.

\begin{algorithm}
\caption{BPC with system identification}
\label{alg3}
\begin{algorithmic}[1]
    \STATE \textbf{Input:} $H, T_0, T, \{ \eta_t \}, \delta, \mathcal{M}, \mathbb{K}, K$
    \STATE \textbf{Phase 1:} Run Algorithm \ref{alg:sys-id} with a budget of $T_0$ to obtain system estimates $\hat{A}, \hat{B}$
    \STATE\textbf{Phase 2:} Run Algorithm \ref{alg2} with the dynamics $\hat{A}, \hat{B}$ for $T - T_0$ timesteps, and $\hat{w}_{T_0} = x_{T_0+1}$
\end{algorithmic}
\end{algorithm}

The performance guarantee for algorithm \ref{alg3} is given in the following theorem. Note that $\hat{\delta}$ is the probability of failure for algorithm \ref{alg:sys-id}.

\begin{theorem} \label{theorem3}
If our system satisfies the assumptions put forth, setting $T_0 = \Theta \left(T^{2/3} \log\hat{\delta}^{-1}\right)$, $\hat{\delta} = \Theta(T^{-1})$, and $\eta_t, \delta$, and $H$ as in theorem \ref{thm:control-known}, we have that the regret incurred by Algorithm \ref{alg3} satisfies
\begin{align*}
    \text{\emph{Regret}} \leq \mathcal{O}\left(T^{3/4}\log^{5/2}T\right). 
\end{align*}
\end{theorem}

\begin{proof}
We split the regret incurred by algorithm \ref{alg3}, which we will denote by $\mathcal{A}$, into
\begin{align*}
    \regret = \text{Regret}_1 + \text{Regret}_2 + \text{Regret}_3
\end{align*}
where the first term corresponds to the regret of the system identification phase, the second term to the regret of algorithm \ref{alg2} relative to the optimal DAC policy $M^\star$, and the final term to the difference between the performance of $M^\star$ on the estimated and true dynamics. Specifically, for $M^{\star} \doteq \arg \min\limits_{M \in \mathcal{M}} \left[ J(M|A, B, w) \right]$ we have
\begin{align}
    \text{Regret}_1 & = J_{T_0}(\text{System identification}) \\ 
    \text{Regret}_2 & = J_{T-T_0}(\mathcal{A} |\hat{A}, \hat{B}, \hat{w}) - J_{T-T_0}(M^{\star}|\hat{A}, \hat{B}, \hat{w}) \label{reg_2}\\
    \text{Regret}_3 & = J_{T-T_0}(M^{\star}|\hat{A}, \hat{B}, \hat{w}) - J_{T-T_0}(M^{\star}|A, B, w). \label{reg_3}
\end{align}
By Lemma 20 in \cite{hazan2019nonstochastic}, the cost incurred during the system identification phase adds up to $\text{Regret}_1 = \mathcal{O}(T_0) = \mathcal{O}(T^{2/3} \log \hat{\delta}^{-1}) = \mathcal{O}(T^{2/3} \log T)$, and since the regret incurred by the second phase of the algorithm has an $\mathcal{O}(T^{3/4}\log^{5/2}T)$ bound, Regret$_1$ is insignificant to our final result.

Next, since $J(M^{\star}|\hat{A}, \hat{B}, \hat{w}) \geq \min\limits_{M \in \mathcal{M}} J(M|\hat{A}, \hat{B}, \hat{w})$ and phase 2 corresponds to running Algorithm \ref{alg2} on $\hat{A}, \hat{B}$ by the Simulation Lemma, Theorem \ref{thm:control-known} implies
\begin{align*}
    \text{Regret}_2 \leq \mathcal{O}\left(T^{3/4}\log^{5/2}T\right)
\end{align*}

We now move on to Regret$_3$. Let $A, B$ denote the true, unknown dynamics and let $\hat{A}, \hat{B}$ be output of Phase 1 after $T_0$ exploration rounds. By Theorem 19 in \cite{hazan2019nonstochastic}, with probability $1 - \hat{\delta}$, we have that 
\begin{align}
    \norm{A - \hat{A}}_F, \norm{B - \hat{B}}_F \leq \eps_{A, B} \label{eqn:sysid}
\end{align}
where $T_0 = \Theta\left(\eps^{-2}_{A, B} \log{\hat{\delta}^{-1}}\right)$. Our choice of $T_0$ therefore implies that $\eps_{A, B} = \Theta\left( T^{-1/3}\log^{-1/2}{\hat{\delta}^{-1}}\right)$. Now, by our assumptions on the bound on the perturbations there exists a constant $\eps_w$ such that $\norm{w_t - \hat{w}_t} \leq \eps_w$. Observe that if $\hat{A}, \hat{B}$ satisfy \ref{eqn:sysid}, then 
\begin{align*}
    \norm{w_t - \hat{w}_t} = \:\:\: &\norm{(x_{t+1} - Ax_t - Bu_t) - (x_{t+1} - \hat{A}x_t - \hat{B}u_t)} & \\
    \leq & \norm{A - \hat{A}} \cdot \norm{x_t} + \norm{B - \hat{B}} \cdot \norm{u_t}  & \text{($\bigtriangleup$-inequality)} \\
    = \:\:\: & \mathcal{O}(\eps_{A, B}) &
\end{align*}
since by assumption $x_t$ and $u_t$ are bounded, which means that the smallest value for $\eps_w$ satisfies $\eps_w = \mathcal{O}(\eps_{A,B})$. By Lemma 17 in \cite{hazan2019nonstochastic} and the formula of state evolution, it follows that for any $M \in \mathcal{M}$:
\begin{align*}
    |J(M|\hat{A},\hat{B},\hat{w})- J(M|A,B,w)| 
    &\leq |J(M|\hat{A},\hat{B},\hat{w})- J(M|A,B,\hat{w})| + |J(M|A, B,\hat{w})- J(M|A,B,w)| \\
    &\leq \mathcal{O}( T(\eps_w + \eps_{A,B}) ) \\
    &\leq \mathcal{O}(T^{2/3}\log^{-1/2}{\hat{\delta}^{-1}}) \label{reg_1}
\end{align*}
with probability $1 - \hat{\delta}$, and hence Regret$_3 = \mathcal{O}(T^{2/3})$ with probability $1 - \hat{\delta}$ as well.

Adding up everything we get that with probability $1 - \hat{\delta}$
\begin{align*}
    \text{Regret} \leq \mathcal{O}\left( T^{2/3}\log\hat{\delta}^{-1} +  T^{3/4} \log^{5/2}T + T^{2/3}\log^{-1/2}{\hat{\delta}^{-1}} \right).
\end{align*}
With at most probability $\hat{\delta}$ we obtain worst-case regret of $\mathcal{O}(T)$ since our costs are bounded. Thus we can set $\hat{\delta} = \Theta(T^{-1})$ and obtain our final regret bound
\begin{align*}
\text{Regret} &\leq \mathcal{O}\left(T^{2/3} \log\hat{\delta}^{-1} + T^{3/4}\log^{5/2}T + + T^{2/3}\log^{-1/2}{\hat{\delta}^{-1}} + \hat{\delta} T\right) \\
&\leq \mathcal{O}(T^{3/4}\log^{5/2}T).
\end{align*}
\end{proof}

\begin{remark}
We see that, for our approach, Algorithm \ref{alg3} enjoys the same regret bound as Algorithm \ref{alg2} despite acting in an unknown system. This is because both the regret incurred during exploration and the difference in performance between the $\hat{A}, \hat{B}$-optimal DAC and the true optimal DAC are of lower order than the regret incurred by Algorithm \ref{alg2}.
\end{remark}

\section{Experimental Results}\label{sec:experiments}

We now provide empirical results of our algorithms' performance on different dynamical systems and under various noise distributions. In all figures, we average the results obtained over 25 runs and include the corresponding confidence intervals. All algorithm implementations are available at \cite{TigerControl}.

\subsection{Control with known dynamics}
We first evaluate our Algorithm \ref{alg2} (BPC) while comparing to GPC \cite{agarwal2019online}, as well as the baseline method Linear Quadratic Regulator (LQR) \cite{kalman1960general}. For both BPC and GPC we initialize $K$ to be the infinite-horizon LQR solution given dynamics $A$ and $B$ in all of the settings below in order to observe the improvement provided by the two perturbation controllers relative to the classical approach.

We consider four different loss functions: 
\begin{enumerate}
    \item $L_2^2$-norm: $c_t(x, u) = ||x||^2 +  ||u||^2$ (also known as \textit{quadratic cost}), 
    \item $L_1$-norm: $c_t(x, u) = ||x||_1 + ||u||_1$, 
    \item $L_{\infty}$-norm: $c_t(x, u) = ||x||_{\infty} + ||u||_{\infty}$,
    \item ReLU: $c_t(x, u) = ||\max(0, x)||_1 + ||\max(0, u)||_1$ (each max taken element-wise).
\end{enumerate}
We run the algorithms on two different linear dynamical system, the $n=2, m=1$ double integrator system defined by $A = \begin{bmatrix} 1 & 1 \\ 0 & 1 \end{bmatrix}$ and $B = \begin{bmatrix} 0  \\ 1\end{bmatrix}$, as well as one additional setting on a larger LDS with $n=5, m=3$ for sparse but non-trivial $A$ and $B$. We analyze the performance of our algorithms for the following 3 noise specifications.

\begin{enumerate}
    \item \textbf{Sanity check.} We run our algorithms with i.i.d Gaussian noise terms $w_t \sim \mathcal{N}(0, I)$. We see that decaying learning rates allow the GPC and BPC to converge to the LQR solution which is optimal for this setup.
    \item \textbf{Sinusoidal noise.} In this setup, we look at the sinusoidal $w_t = sin(t / (20\pi))$. In this correlated noise setting, the LQR policy is sub-optimal, and we see that both BPC and GPC outperform it.
    \item \textbf{Gaussian random walk.} In the Gaussian random walk setting, each noise term is distributed normally, with the previous noise term as its mean, i.e. $w_{t+1} = \mathcal{N}\left(w_t, \frac{1}{T}\right)$. Since $T = 1000$, we have approximately that $w_{t+1} - w_t \sim \mathcal{N} (0, 0.3^2)$.
\end{enumerate}

\begin{figure}[htp!]
    \centering
    \begin{tabular}{cc}
    \subfloat[Sanity check with fixed learning rate]{\includegraphics[width=0.24\textwidth] {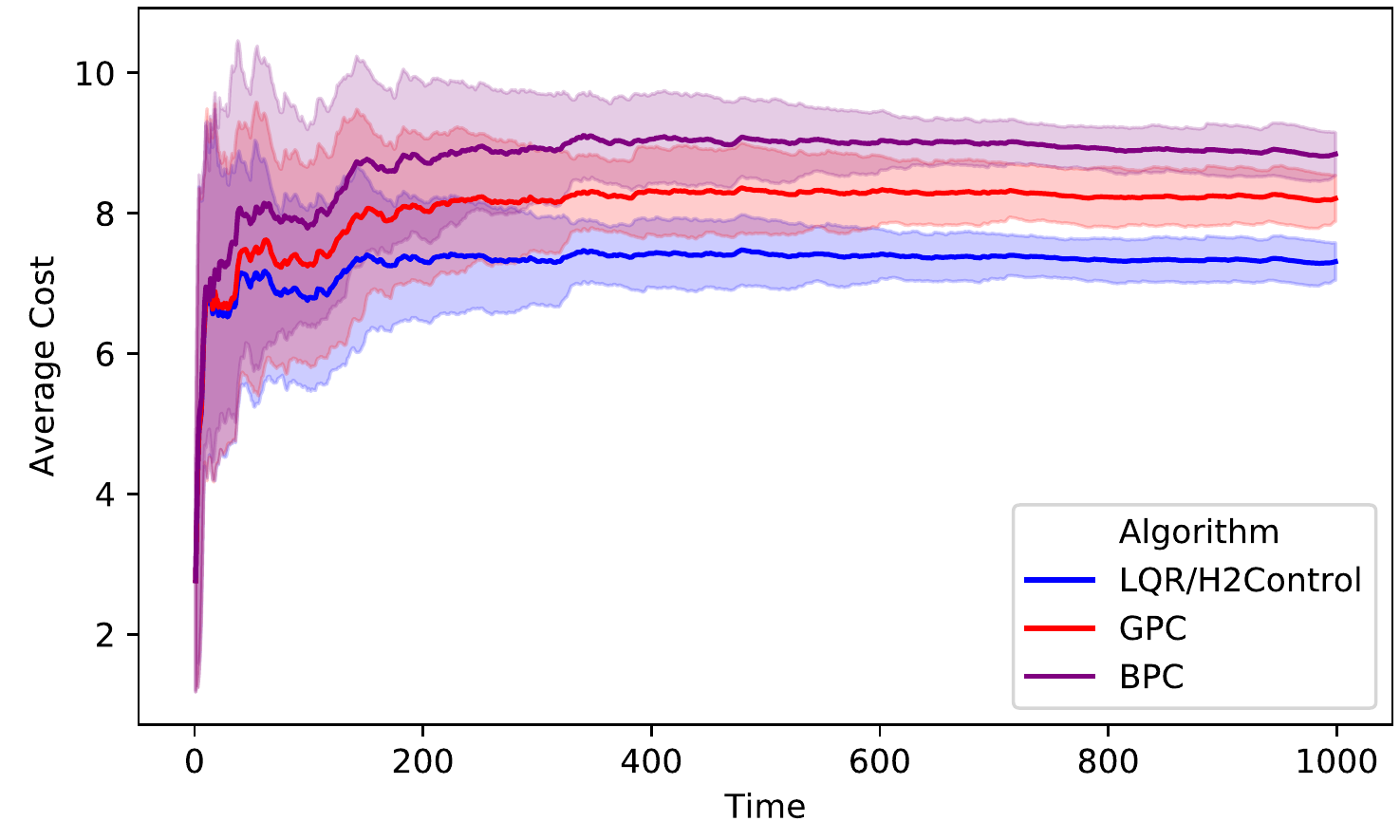}}
    \enspace
    \subfloat[Sanity check with decaying learning rate.]{\includegraphics[width=0.24\textwidth] {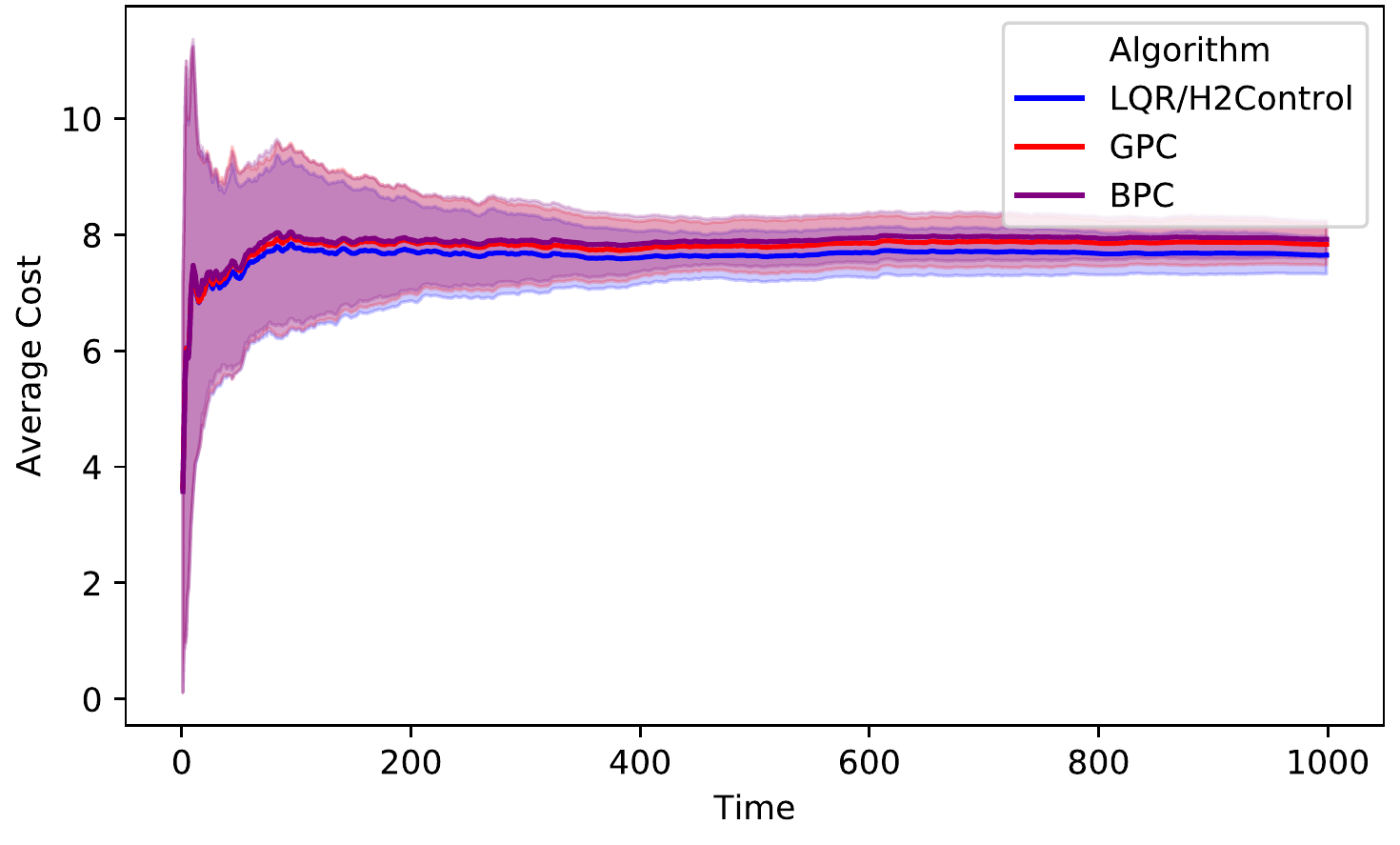}}
    \enspace
    \subfloat[Sinusoidal noise and quadratic costs.]{\includegraphics[width=0.25\textwidth]{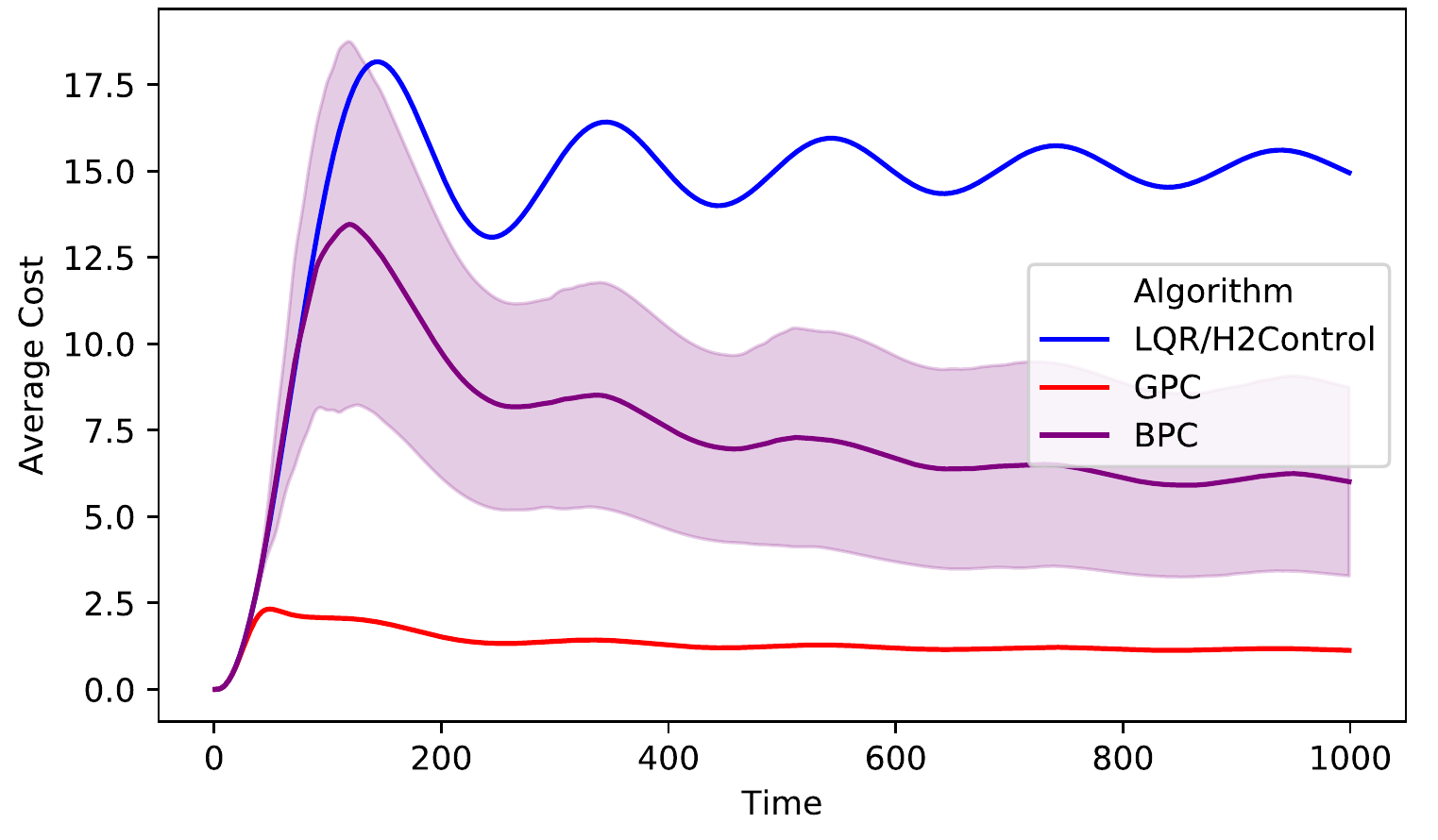}}
    \enspace
    \subfloat[Sinusoidal noise and $L_1$ costs.]{\includegraphics[width=0.22\textwidth] {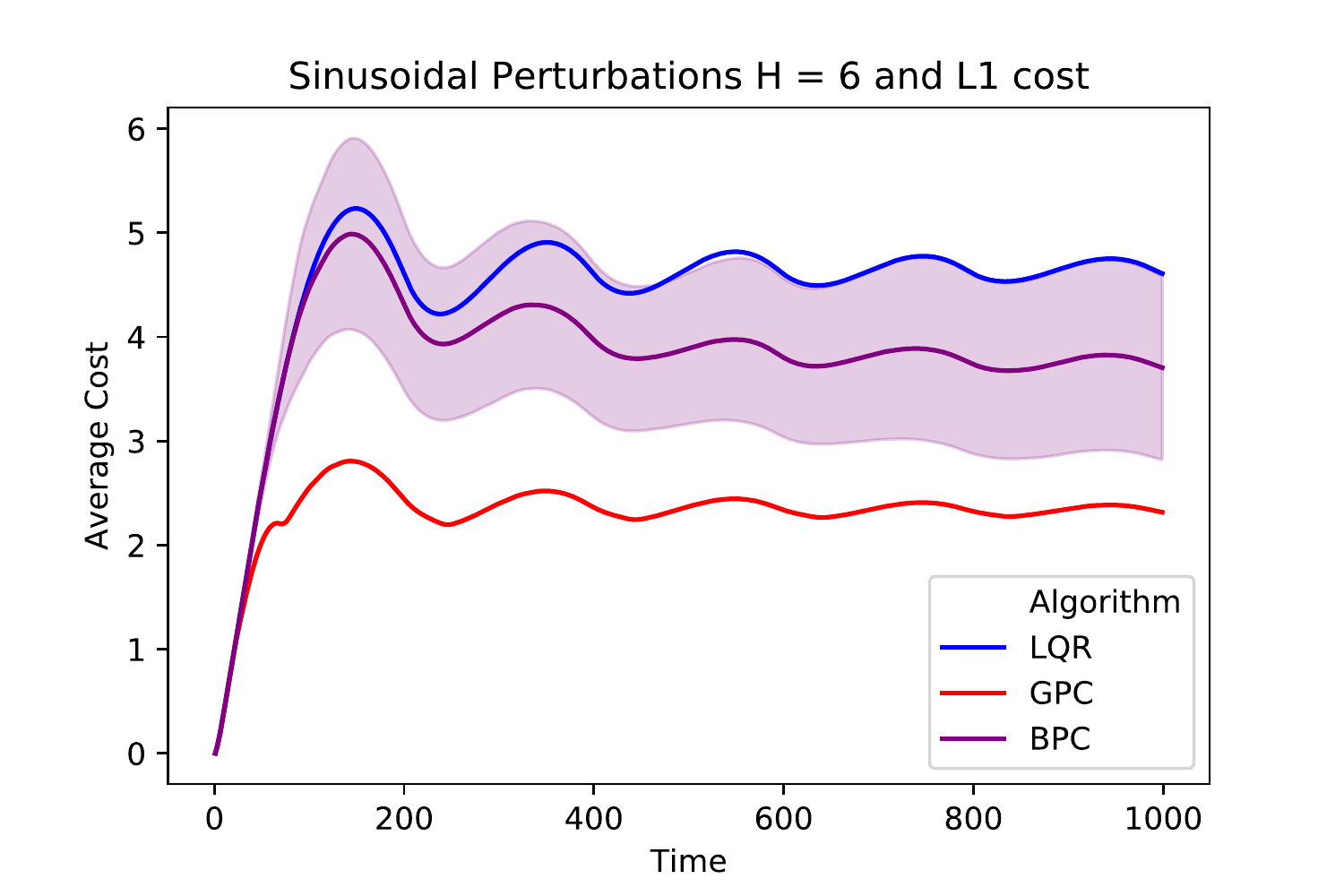}}
    \enspace \\
    \subfloat[Random walk noise on simple LDS and quadratic costs.]{\includegraphics[width=0.25\textwidth] {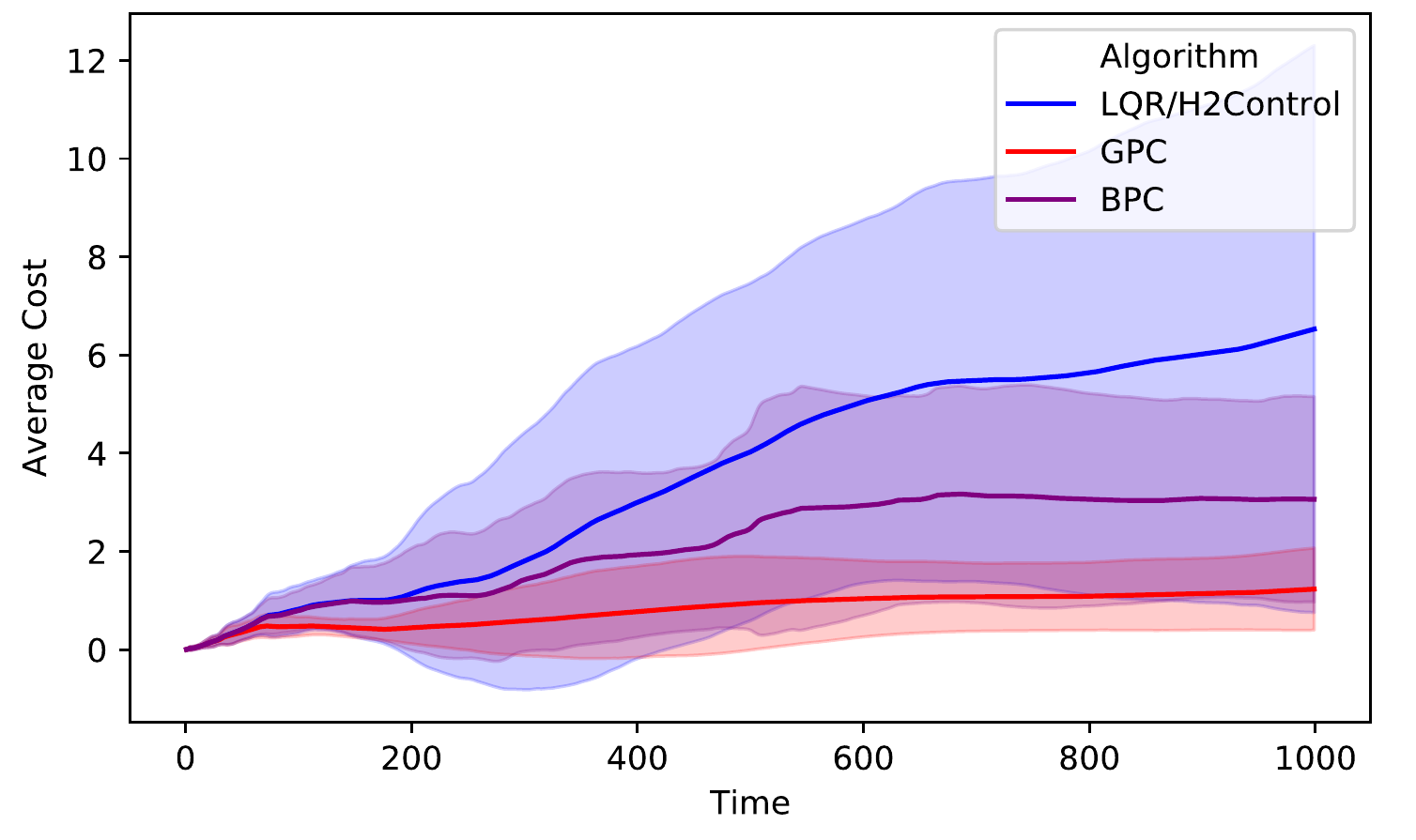}}
    \enspace
    \subfloat[Random walk noise on simple LDS and $L_1$ costs.]{\includegraphics[width=0.23\textwidth] {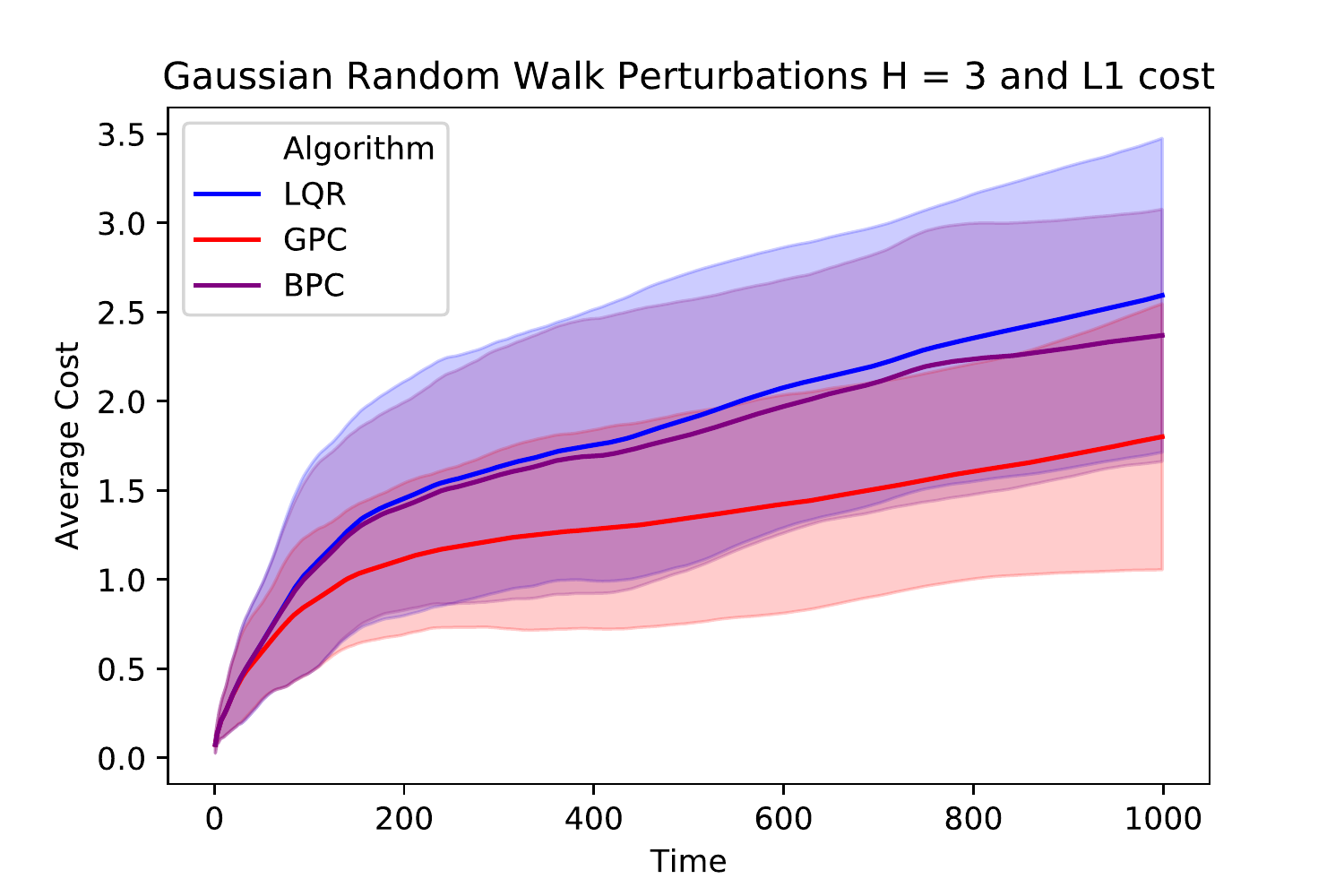}}
    \enspace
    \subfloat[Sinusoidal noise on complex LDS and $L_{\infty}$ cost.]{\includegraphics[width=0.23\textwidth] {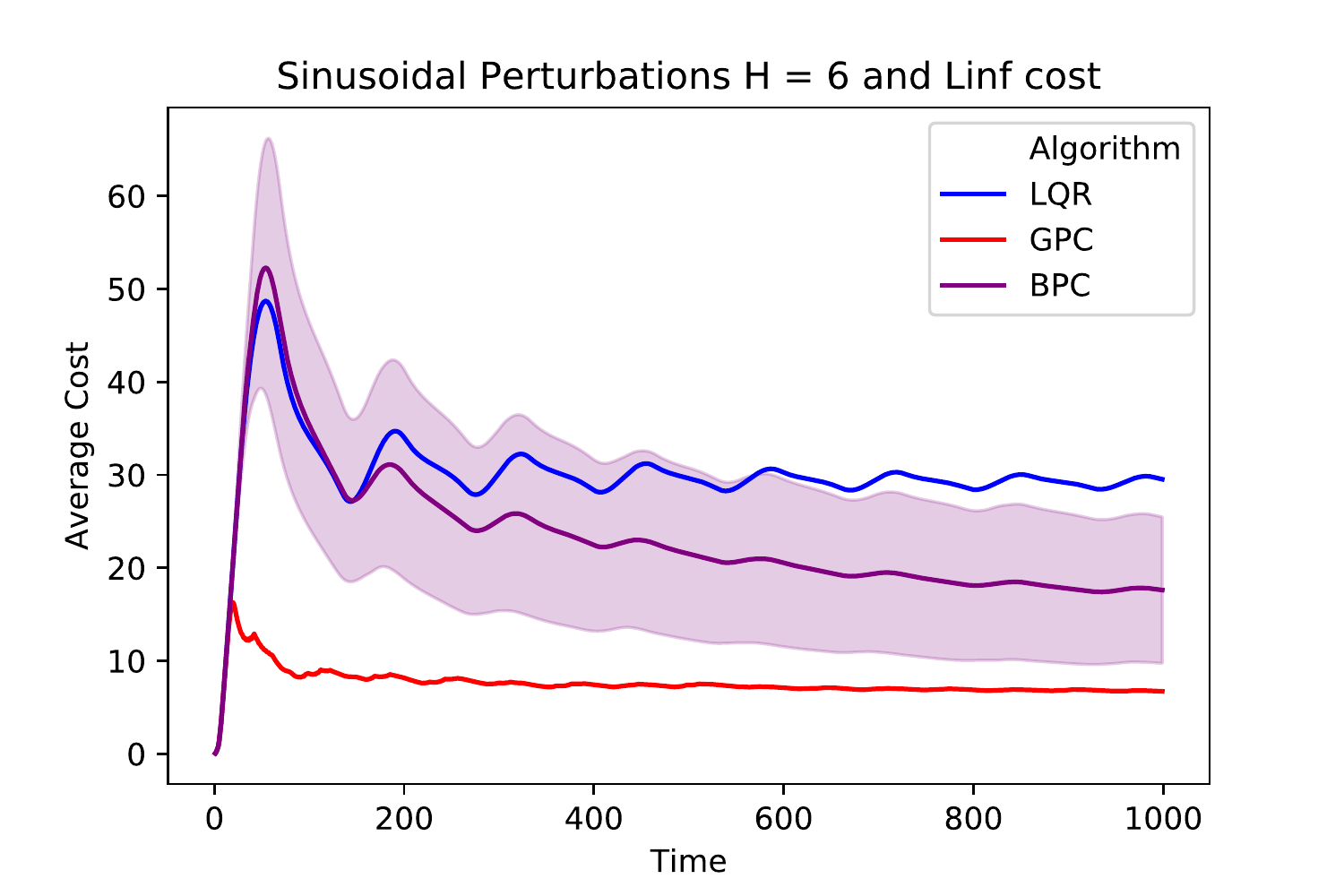}} 
    \enspace
    \subfloat[Sinusoidal noise on complex LDS and ReLU costs.]{\includegraphics[width=0.23\textwidth] {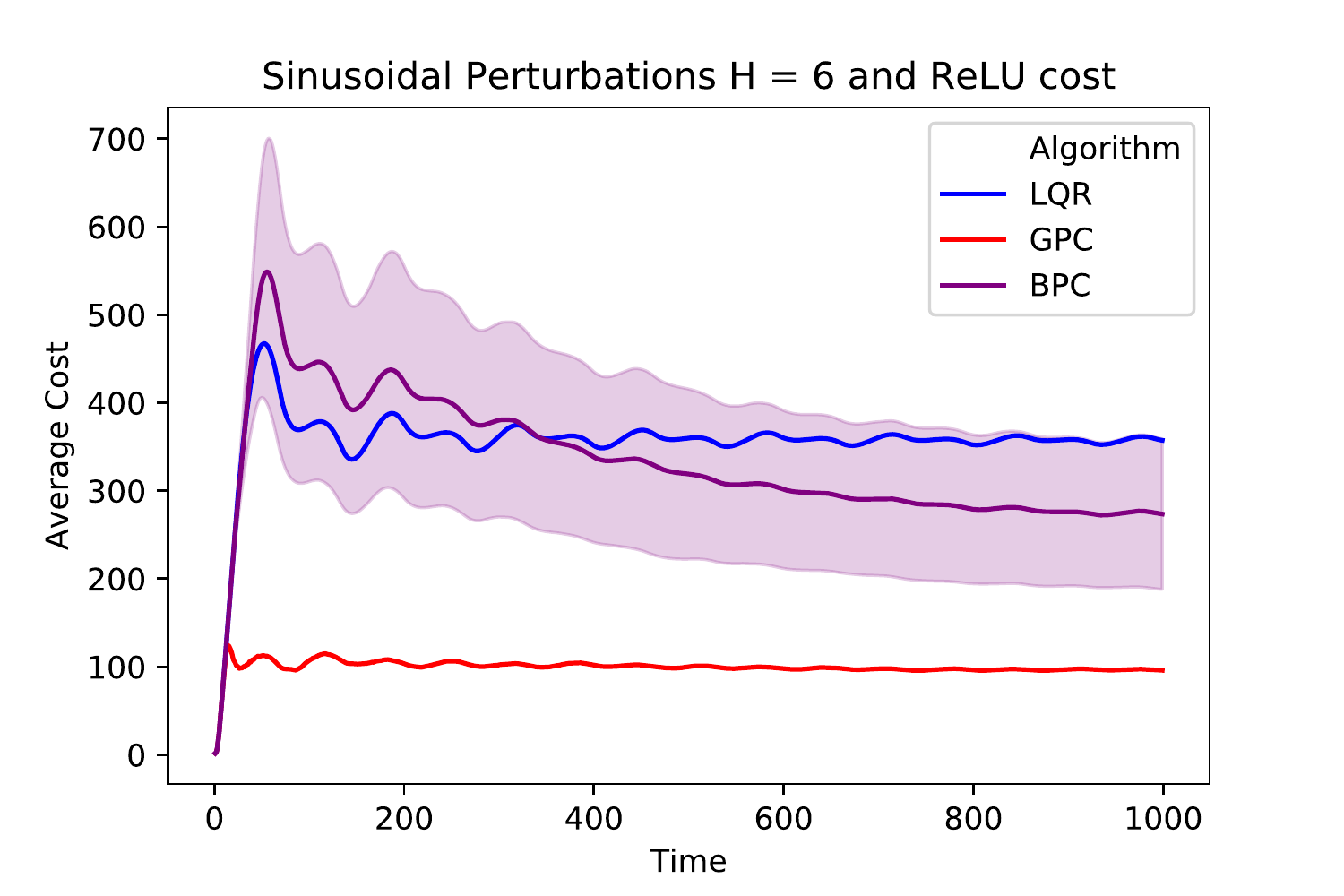}}
    \end{tabular}
    \caption{Known dynamics, small and large LDS setting.}
\end{figure}

\subsection{Control with unknown dynamics}
Next, we evaluate algorithm \ref{alg3} on unknown dynamical systems. We obtain estimates $\hat{A}$ and $\hat{B}$ of the system dynamics using two different types of system identification methods, the first being the method described in algorithm \ref{alg:sys-id}, and the second being regular linear regression based on all observations up to the current time point. We then proceed with the experiments as in the previous section, with all algorithms being given $\hat{A}$ and $\hat{B}$ instead of the true $A, B$. That is, LQR produces policy $\hat{K}$ based on the solution of the algebraic Riccati equation given by $\hat{A}$ and $\hat{B}$, and both BPC and GPC start from this initial $\hat{K}$ and obtain estimates of the disturbances $\hat{w}_t$ based on the approximate dynamics. 

We run experiments with quadratic costs for the first LDS described in the known dynamics section, with scaled down dynamics matrices $A$ and $B$ such that their nuclear norm is strictly less than 1 so that the dynamical system remains stable during the identification phase. The system identification phase is repeated for each experiment and runs for $T_0 = 5000$ time steps and with initial control matrix $K$ set to 0.

\begin{figure}[H]
    \centering
    \begin{tabular}{cc}
    \subfloat[Sanity check]{\includegraphics[width=0.25\textwidth] {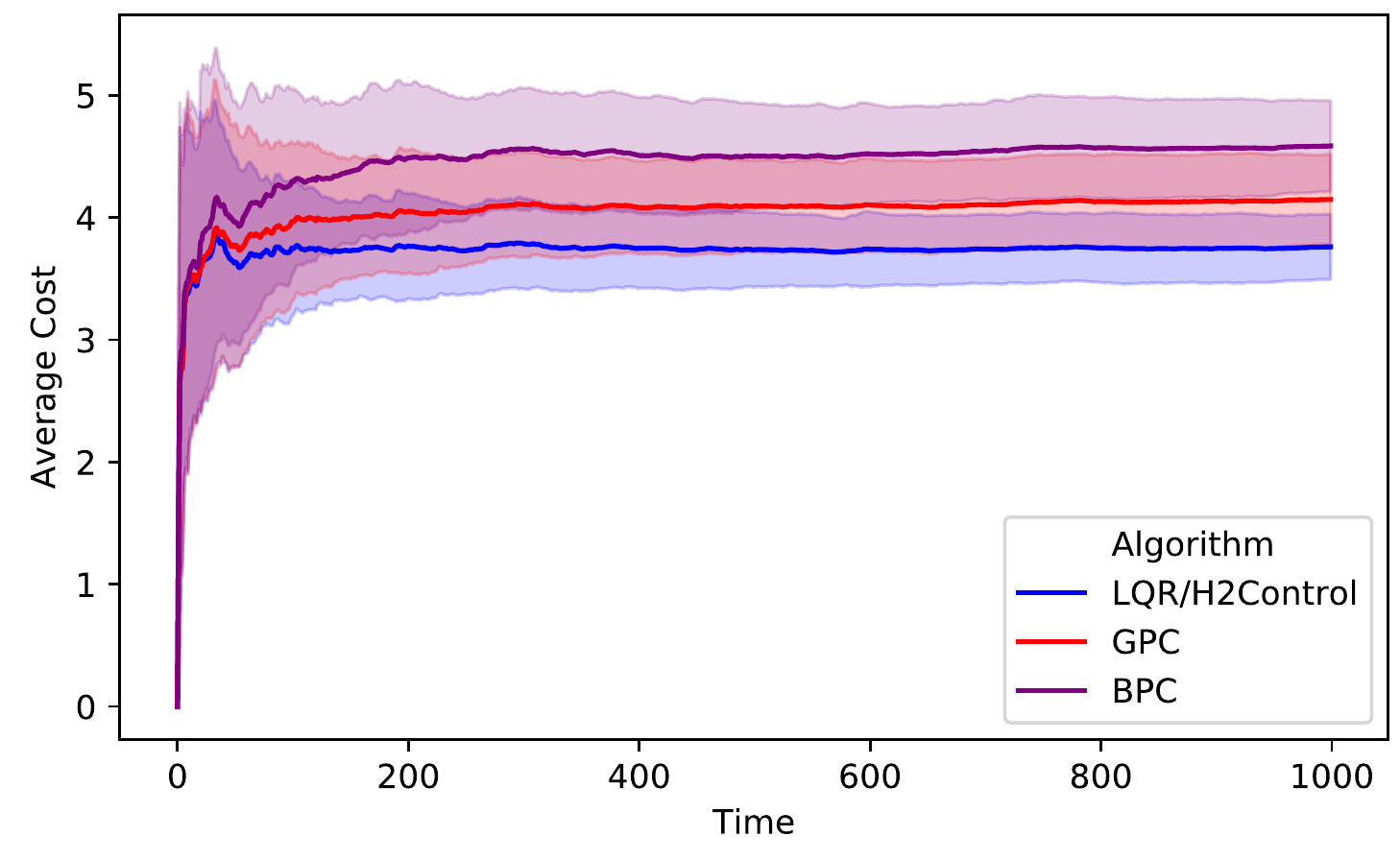}}
    \enspace
    \subfloat[Sinusoidal noise]{\includegraphics[width=0.25\textwidth] {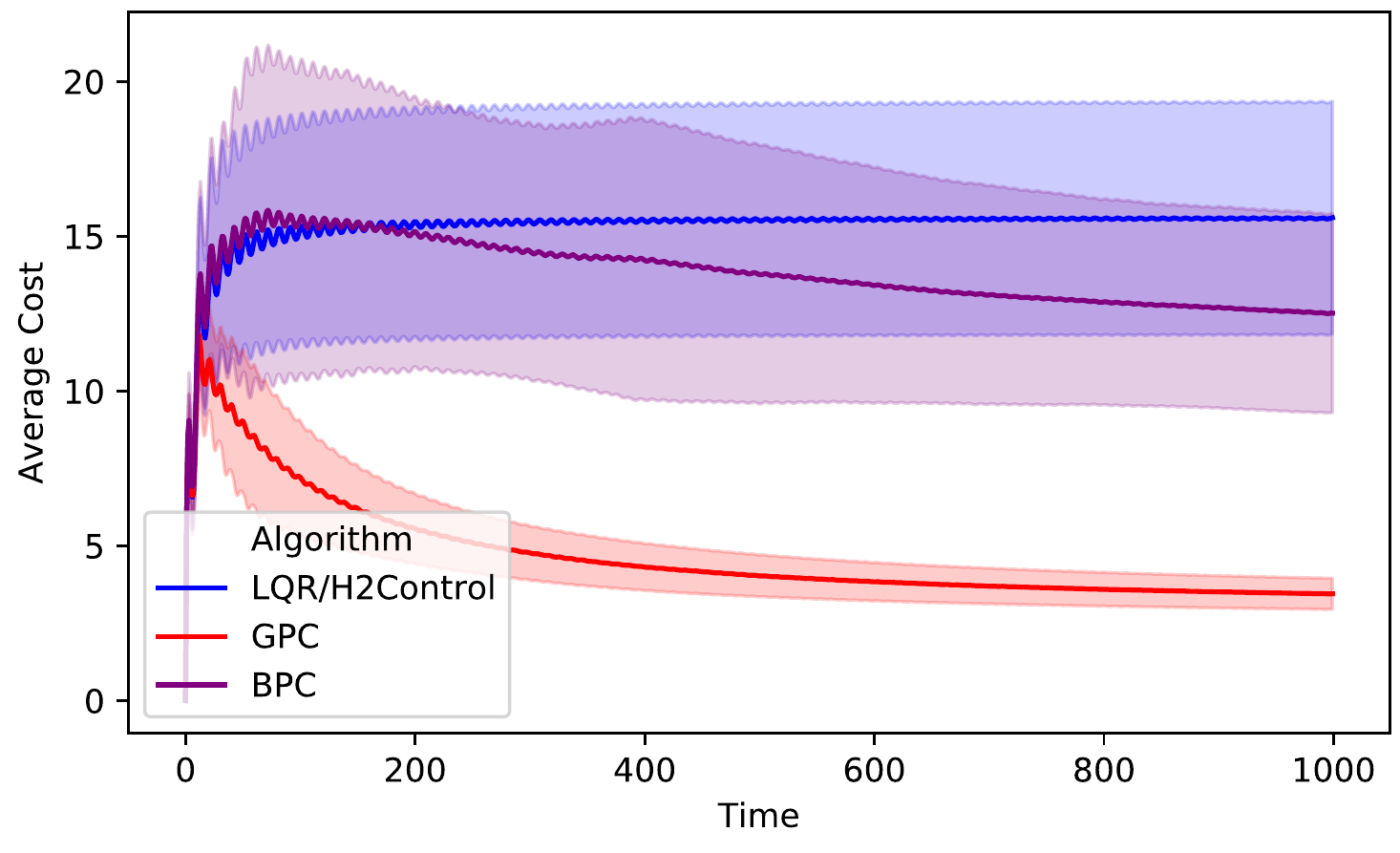}}
    \enspace
    \subfloat[Random walk noise]{\includegraphics[width=0.25\textwidth] {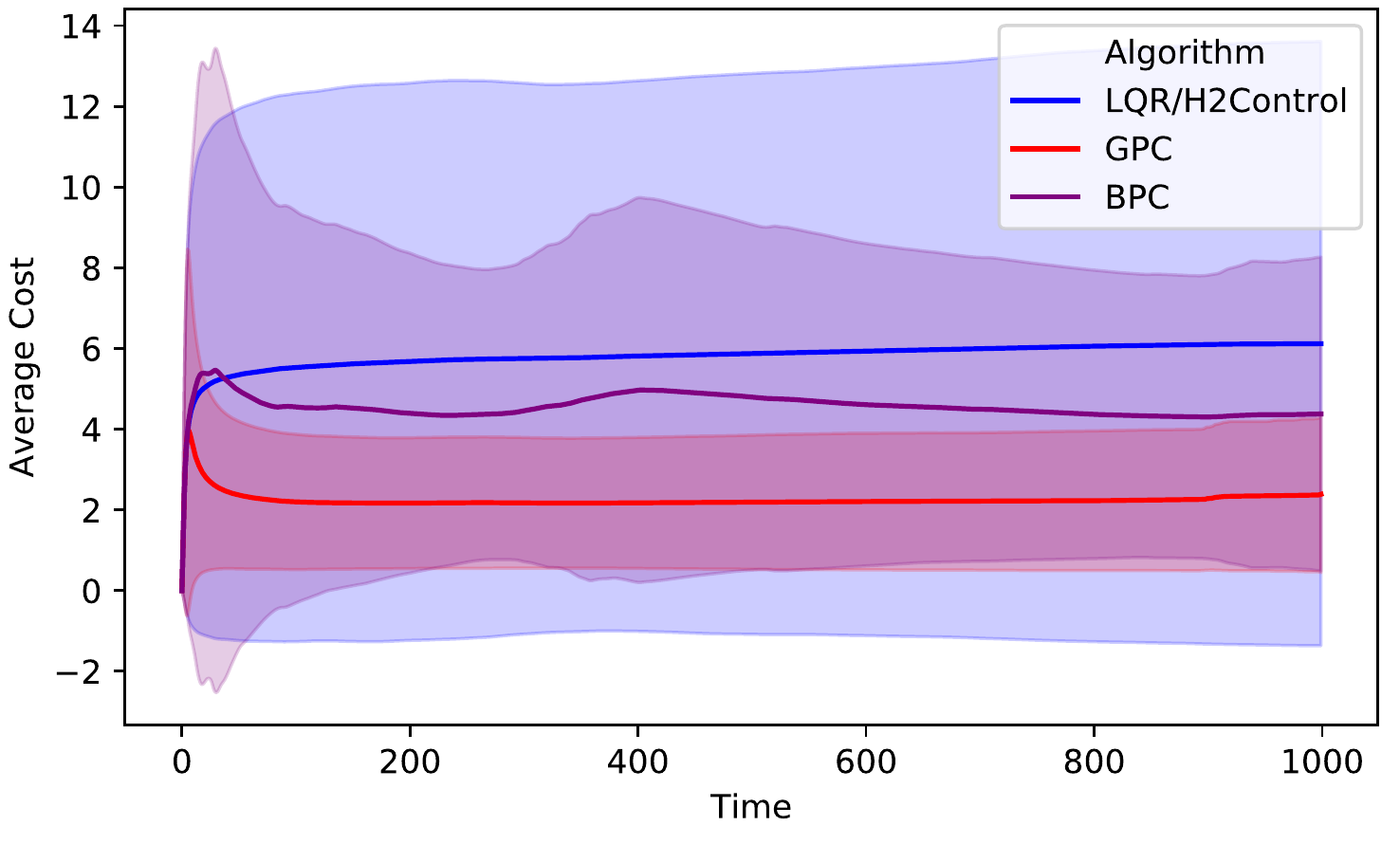}}
    \\
    \subfloat[Sanity check]{\includegraphics[width=0.25\textwidth] {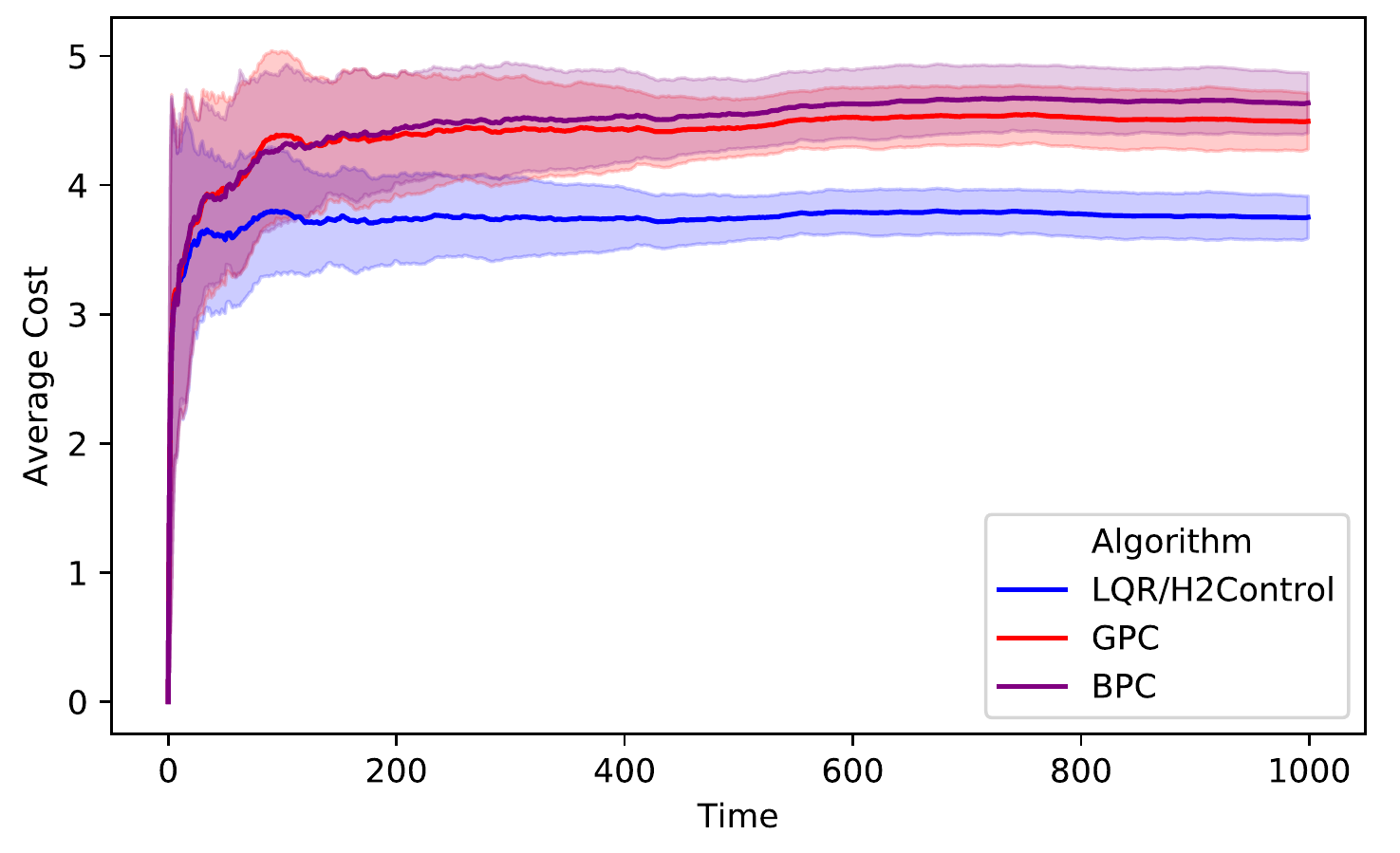}}
    \enspace
    \subfloat[Sinusoidal noise]{\includegraphics[width=0.25\textwidth] {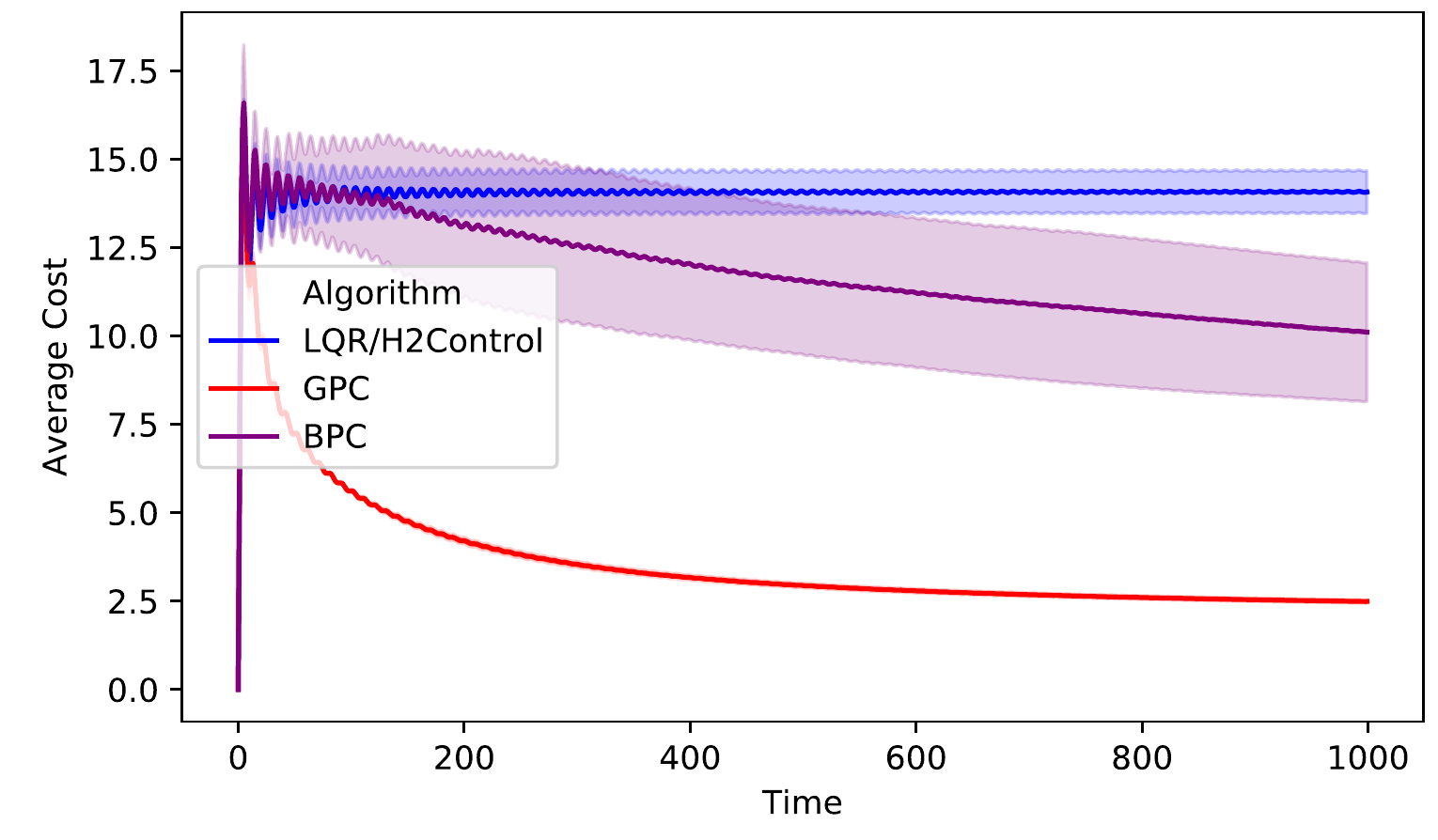}}
    \enspace
    \subfloat[Random walk noise]{\includegraphics[width=0.25\textwidth] {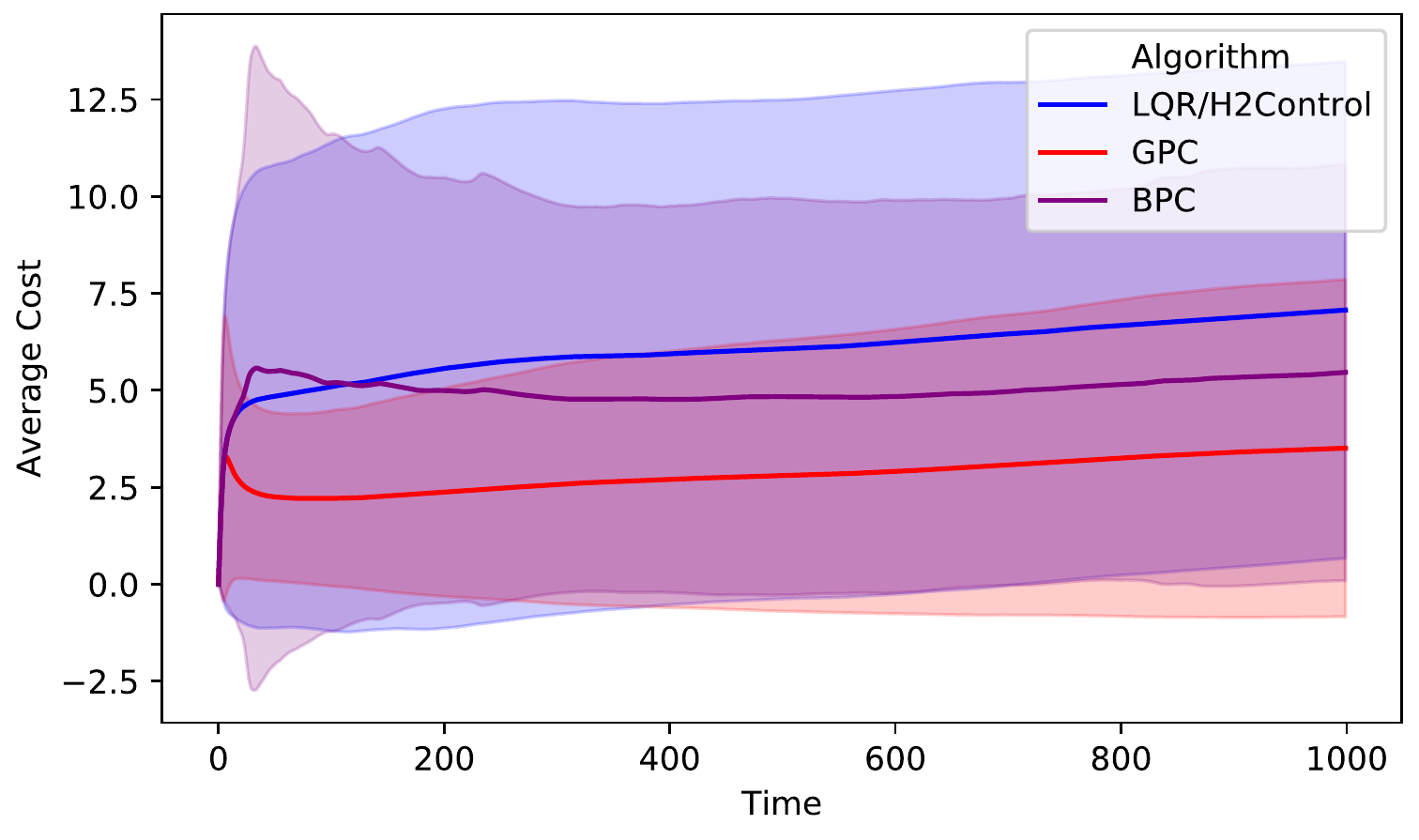}}
    \end{tabular}
    \caption{Unknown dynamics, the top row uses the system identification method in algorithm \ref{alg:sys-id}, and the bottom row uses linear regression.}
\end{figure}

\section{Conclusions and Open Questions}

We have considered the non-stochastic control problem with the additional difficulty of learning with only bandit feedback. We give an efficient method with sublinear regret for this challenging problem in the case of linear dynamics based upon a new algorithm for bandit convex optimization with memory, which may be of independent interest.  The application of bandit optimization to control is complicated due to time dependency issues, which required introducing an artificial delay in our online learning method. 

The setting of control with general convex losses was proposed in 1987 by Tyrrell Rockafellar \cite{rockafellar1987linear} in order to handle constraints on state and control. It remains open to add constraints (such as safety constaints) to online nonstochastic control. 
Other questions that remain open are quantitative: the worst case attainable regret bounds can be potentially improved to $\sqrt{T}$. The dependence on the system dimensions can also be tightened.

%%% \nocite{langley00}

\bibliography{cited.bib}

\begin{thebibliography}{10}

\bibitem{abbasi2014tracking}
Yasin Abbasi-Yadkori, Peter Bartlett, and Varun Kanade.
\newblock Tracking adversarial targets.
\newblock In {\em International Conference on Machine Learning}, pages
  369--377, 2014.

\bibitem{a2}
Yasin Abbasi-Yadkori, Nevena Lazic, and Csaba Szepesv{\'a}ri.
\newblock Model-free linear quadratic control via reduction to expert
  prediction.
\newblock In {\em The 22nd International Conference on Artificial Intelligence
  and Statistics}, pages 3108--3117, 2019.

\bibitem{abbasi2011regret}
Yasin Abbasi-Yadkori and Csaba Szepesv{\'a}ri.
\newblock Regret bounds for the adaptive control of linear quadratic systems.
\newblock In {\em Proceedings of the 24th Annual Conference on Learning
  Theory}, pages 1--26, 2011.

\bibitem{agarwal2019online}
Naman Agarwal, Brian Bullins, Elad Hazan, Sham~M. Kakade, and Karan Singh.
\newblock Online control with adversarial disturbances, 2019.

\bibitem{agarwal2019logarithmic}
Naman Agarwal, Elad Hazan, and Karan Singh.
\newblock Logarithmic regret for online control.
\newblock {\em arXiv preprint arXiv:1909.05062}, 2019.

\bibitem{anava2013online}
Oren Anava, Elad Hazan, and Shie Mannor.
\newblock Online convex optimization against adversaries with memory and
  application to statistical arbitrage.
\newblock {\em arXiv preprint arXiv:1302.6937}, 2013.

\bibitem{arora2018towards}
Sanjeev Arora, Elad Hazan, Holden Lee, Karan Singh, Cyril Zhang, and Yi~Zhang.
\newblock Towards provable control for unknown linear dynamical systems.
\newblock {\em International Conference on Learning Representations}, 2018.
\newblock rejected: invited to workshop track.

\bibitem{cohen2018online}
Alon Cohen, Avinatan Hassidim, Tomer Koren, Nevena Lazic, Yishay Mansour, and
  Kunal Talwar.
\newblock Online linear quadratic control.
\newblock {\em arXiv preprint arXiv:1806.07104}, 2018.

\bibitem{cohen2019learning}
Alon Cohen, Tomer Koren, and Yishay Mansour.
\newblock Learning linear-quadratic regulators efficiently with only $\sqrt(t)$
  regret.
\newblock {\em arXiv preprint arXiv:1902.06223}, 2019.

\bibitem{dean2018regret}
Sarah Dean, Horia Mania, Nikolai Matni, Benjamin Recht, and Stephen Tu.
\newblock Regret bounds for robust adaptive control of the linear quadratic
  regulator.
\newblock In {\em Advances in Neural Information Processing Systems}, pages
  4188--4197, 2018.

\bibitem{even2009online}
Eyal Even-Dar, Sham~M Kakade, and Yishay Mansour.
\newblock Online {M}arkov decision processes.
\newblock {\em Mathematics of Operations Research}, 34(3):726--736, 2009.

\bibitem{faradonbeh2018finite}
Mohamad Kazem~Shirani Faradonbeh, Ambuj Tewari, and George Michailidis.
\newblock Finite time identification in unstable linear systems.
\newblock {\em Automatica}, 96:342--353, 2018.

\bibitem{fazel2018global}
Maryam Fazel, Rong Ge, Sham~M Kakade, and Mehran Mesbahi.
\newblock Global convergence of policy gradient methods for the linear
  quadratic regulator.
\newblock {\em arXiv preprint arXiv:1801.05039}, 2018.

\bibitem{flaxman2004online}
Abraham~D. Flaxman, Adam~Tauman Kalai, and H.~Brendan McMahan.
\newblock Online convex optimization in the bandit setting: gradient descent
  without a gradient, 2004.

\bibitem{hazan201210}
Elad Hazan.
\newblock The convex optimization approach to regret minimization.
\newblock {\em Optimization for machine learning}, 2012.

\bibitem{hazan2019nonstochastic}
Elad Hazan, Sham~M Kakade, and Karan Singh.
\newblock The nonstochastic control problem.
\newblock {\em arXiv preprint arXiv:1911.12178}, 2019.

\bibitem{hazan2018spectral}
Elad Hazan, Holden Lee, Karan Singh, Cyril Zhang, and Yi~Zhang.
\newblock Spectral filtering for general linear dynamical systems.
\newblock In {\em Advances in Neural Information Processing Systems}, pages
  4634--4643, 2018.

\bibitem{hazan2017learning}
Elad Hazan, Karan Singh, and Cyril Zhang.
\newblock Learning linear dynamical systems via spectral filtering.
\newblock In {\em Advances in Neural Information Processing Systems}, pages
  6702--6712, 2017.

\bibitem{kalman1960general}
Rudolf~E Kalman.
\newblock On the general theory of control systems.
\newblock In {\em Proceedings First International Conference on Automatic
  Control, Moscow, USSR}, 1960.

\bibitem{anima2}
Sahin Lale, Kamyar Azizzadenesheli, Babak Hassibi, and Anima Anandkumar.
\newblock Logarithmic regret bound in partially observable linear dynamical
  systems, 2020.

\bibitem{anima1}
Sahin Lale, Kamyar Azizzadenesheli, Babak Hassibi, and Anima Anandkumar.
\newblock Regret bound of adaptive control in linear quadratic gaussian (lqg)
  systems, 2020.

\bibitem{anima3}
Sahin Lale, Kamyar Azizzadenesheli, Babak Hassibi, and Anima Anandkumar.
\newblock Regret minimization in partially observable linear quadratic control,
  2020.

\bibitem{mania2019certainty}
Horia Mania, Stephen Tu, and Benjamin Recht.
\newblock Certainty equivalent control of lqr is efficient.
\newblock {\em arXiv preprint arXiv:1902.07826}, 2019.

\bibitem{neu2010online}
Gergely Neu, Andras Antos, Andr{\'a}s Gy{\"o}rgy, and Csaba Szepesv{\'a}ri.
\newblock Online markov decision processes under bandit feedback.
\newblock In {\em Advances in Neural Information Processing Systems}, pages
  1804--1812, 2010.

\bibitem{oymak2019non}
Samet Oymak and Necmiye Ozay.
\newblock Non-asymptotic identification of lti systems from a single
  trajectory.
\newblock In {\em 2019 American Control Conference (ACC)}, pages 5655--5661.
  IEEE, 2019.

\bibitem{TigerControl}
Google~AI Princeton.
\newblock Tigercontrol.
\newblock \url{https://github.com/MinRegret/TigerControl}, 2020.

\bibitem{rockafellar1987linear}
R~Tyrell Rockafellar.
\newblock Linear-quadratic programming and optimal control.
\newblock {\em SIAM Journal on Control and Optimization}, 25(3):781--814, 1987.

\bibitem{sarkar2019near}
Tuhin Sarkar and Alexander Rakhlin.
\newblock Near optimal finite time identification of arbitrary linear dynamical
  systems.
\newblock In {\em International Conference on Machine Learning}, pages
  5610--5618, 2019.

\bibitem{sarkar2019finite}
Tuhin Sarkar, Alexander Rakhlin, and Munther~A Dahleh.
\newblock Finite-time system identification for partially observed lti systems
  of unknown order.
\newblock {\em arXiv preprint arXiv:1902.01848}, 2019.

\bibitem{simchowitz2019learning}
Max Simchowitz, Ross Boczar, and Benjamin Recht.
\newblock Learning linear dynamical systems with semi-parametric least squares.
\newblock {\em arXiv preprint arXiv:1902.00768}, 2019.

\bibitem{simchowitz2018learning}
Max Simchowitz, Horia Mania, Stephen Tu, Michael~I Jordan, and Benjamin Recht.
\newblock Learning without mixing: Towards a sharp analysis of linear system
  identification.
\newblock {\em arXiv preprint arXiv:1802.08334}, 2018.

\bibitem{simchowitz2020improper}
Max Simchowitz, Karan Singh, and Elad Hazan.
\newblock Improper learning for non-stochastic control, 2020.

\bibitem{z2}
Robert~F Stengel.
\newblock {\em Optimal control and estimation}.
\newblock Courier Corporation, 1994.

\bibitem{yu2009markov}
Jia~Yuan Yu, Shie Mannor, and Nahum Shimkin.
\newblock Markov decision processes with arbitrary reward processes.
\newblock {\em Mathematics of Operations Research}, 34(3):737--757, 2009.

\bibitem{z1}
Kemin Zhou, John~Comstock Doyle, Keith Glover, et~al.
\newblock {\em Robust and optimal control}, volume~40.
\newblock Prentice hall New Jersey, 1996.

\end{thebibliography}
\bibliographystyle{plain}

\newpage
\appendix
%\newpage
%\input{full_appendix.tex}

\end{document}